\documentclass[article,onecolumn]{IEEEtran}
\usepackage[left=3cm,right=3cm,top=1.5cm,bottom=2cm,includeheadfoot]{geometry}

\usepackage[utf8]{inputenc} 
\usepackage[T1]{fontenc}    
\usepackage{hyperref}       
\usepackage{url}            
\usepackage{booktabs}       
\usepackage{amsfonts}       
\usepackage{nicefrac}       
\usepackage{microtype}      

\title{
The Restricted Isometry of $\mathrm{ReLU}$ Networks: Generalization through Norm Concentration
}

\newcommand{\tu}{Institute of Mathematics, Technische Universit{\"a}t Berlin, Germany}
\author{
    \IEEEauthorblockN{
        Alex~Goe{\ss}mann\IEEEauthorrefmark{1}\IEEEauthorrefmark{4},
        Gitta~Kutyniok\IEEEauthorrefmark{4}}

    \IEEEauthorblockA{\IEEEauthorrefmark{4}\tu{}}

    \thanks{\IEEEauthorrefmark{1} goessmann@tu-berlin.de
    }
}

\usepackage{amsmath,amsfonts,amssymb,amsthm,bbm,graphicx,enumerate,times}
\usepackage{tikz}
\usetikzlibrary{angles,arrows,calc,shapes,positioning}

\newcommand{\rr}{\mathbb{R}}
\newcommand{\nrip}{\mathrm{NeuRIP_s}}

\usepackage{braket}
\newtheorem{theorem}{Theorem}
\newtheorem{definition}{Definition}
\newtheorem{lemma}{Lemma}
\newtheorem{remark}{Remark}
\newtheorem{assumption}{Assumption}
\newtheorem{corollary}{Corollary}

\begin{document}

\maketitle

\begin{abstract}
While regression tasks aim at interpolating a relation on the entire input space, they often have to be solved with a limited amount of training data.
Still, if the hypothesis functions can be sketched well with the data, one can hope for identifying a generalizing model.

In this work, we introduce with the Neural Restricted Isometry Property ($\nrip$) a uniform concentration event, in which all shallow $\mathrm{ReLU}$ networks are sketched with the same quality.
To derive the sample complexity for achieving $\nrip$, we bound the covering numbers of the networks in the Sub-Gaussian metric and apply chaining techniques. In case of the $\nrip$ event, we then provide bounds on the expected risk, which hold for networks in any sublevel set of the empirical risk. We conclude that all networks with sufficiently small empirical risk generalize uniformly.
\end{abstract}

\section{Introduction}

A central desire for any scientific model is an assessing estimation of its limitation.
In recent years, tools for automated model discovery from given training data have been developed in the area of supervised machine learning.
However, such methods lack of a sophisticated theoretical foundation, which would provide estimates for the limitations of such models.
Statistical learning theory quantifies the limitation of a trained model in terms of the generalization error and introduces for its treatment the $\mathrm{VC}$-dimension \cite{vapnik_nature_2000} and the Rademacher complexity \cite{shalev-schwartz_shai_understanding_2014,mohri_foundations_2018}.
The $\mathrm{VC}$-dimension of neural networks \cite{maass_vapnik-chervonenkis_1995} and its extensions \cite{bartlett_for_1996,bartlett_sample_1998} have led to generalization error bounds for classification problems.
Bounds on the Rademacher complexity of shallow neural networks and their application have been derived in \cite{goos_rademacher_2001,xie_diverse_2017,oymak_learning_2018-1}.
Although these traditional complexity notions were successful in classification problems \cite{bartlett_spectrally-normalized_2017-1}, they do not apply to generic regression problems with unbounded risk functions, as we study in this work.
Moreover, the traditional tools of statistical learning theory fail to provide a satisfying generalization theory of neural networks \cite{neyshabur_search_2014,neyshabur_exploring_2017,zhang_understanding_2017,mucke_global_2019}.

Understanding the risk surface when training neural networks is crucial to develop a theoretical foundation of neural network based machine learning, in particular when aiming to derive an understanding of generalization phenomena.
Recent works on neural networks hint at astonishing properties of the risk surface \cite{goodfellow_qualitatively_2015,soudry_no_2016}.
For large networks, the local minima of the risk build a small bond at the global minimum \cite{choromanska_loss_2015}.
Surprisingly, global minima are present in every connected component of the sublevel set of the risk \cite{venturi_spurious_2019,soltanolkotabi_theoretical_2019,nguyen_connected_2019} and are path-connected \cite{draxler_essentially_2018}.
In this work, we complement these findings towards a generalization theory of shallow $\mathrm{ReLU}$ networks, by providing uniform generalization error bounds in the sublevel set of the empirical risk.
We apply methods from the analysis of convex linear regression problems, where generalization bounds for empirical risk minimizers \cite{plan_generalized_2016,genzel_mismatch_2019,genzel_generic_2020} follow from recent results in the chaining theory of stochastic processes \cite{tropp_convex_2014,dirksen_tail_2015,mendelson_upper_2016}.
For non-convex sets of hypothesis functions the empirical risk minimization can in general not be solved efficiently.
But, under mild assumptions, it is still possible to derive generalization error bounds, as we show in this paper for shallow $\mathrm{ReLU}$ networks.
Existing works \cite{vito_learning_2005,cohen_stability_2013,eigel_convergence_2020} apply methods from the theory of compressed sensing \cite{eldar_compressed_2012,foucart_mathematical_2013} to bound generalization errors for arbitrary hypothesis function sets, but do not capture the stochastic structure of the risk by the more sophisticated chaining theory.

Our paper is organized as follows.
We start with the formulation of our assumptions on the parameters of shallow $\mathrm{ReLU}$ networks and the data distribution to be interpolated in Section~\ref{sec:assumptions}.
The expected and the empirical risk will be introduced in Section~\ref{sec:empiric}, where we further define the Neural Restricted Isometry Property ($\nrip$) as an event of uniform norm concentration.
We provide with Theorem~\ref{thm:NeuRIP} a bound on the sample complexity for achieving $\nrip$, which depends on the network architecture and assumptions on the parameters.
In Section~\ref{sec:generalization}, we show upper bounds on the generalization error that hold uniformly on the sublevel sets of the empirical risk.
In fact, we derive this property both in a network recovery (Theorem~\ref{thm:Prpm}) and an agnostic learning framework (Theorem~\ref{the:Prym}).
If any optimization algorithm identifies a network with a small empirical risk, these results guarantee a small generalization error.
In Section~\ref{sec:size}, we develop the main proof techniques to derive the sample complexity of achieving $\nrip$, which are based on the chaining theory of stochastic processes.
We provide bounds on the Talagrand-functional of shallow $\mathrm{ReLU}$ networks in Lemma~\ref{lem:dudleyshallow}, which we expect to be of independent interest.
The derived results are concluded in Section~\ref{sec:conclusion}, where we further discuss future extensions.

\section{Notation and Assumptions}\label{sec:assumptions}

In the sequel, we will introduce the key notations and assumptions on the neural networks treated in this work.
The \emph{Rectified Linear Unit} ($\mathrm{ReLU}$) function $\phi:\rr\rightarrow \rr$ is defined as $\phi(x):=\max(x,0)$.
Given a weight vector $w\in \rr^d$, a bias $b\in \rr$ and a sign $\kappa \in \{\pm 1\}$ the $\mathrm{ReLU}$ \emph{neuron} is a function $\phi_{(w,b,\kappa)}: \rr^d \rightarrow \rr$ defined as
\begin{align*}
\phi_{(w,b,\kappa)}(x)=\kappa \phi(\braket{w,x}+b) \, . 
\end{align*}
\emph{Shallow neural networks} are weighted sums of neurons.
They are typically depicted by a graph with $n$ neurons in one hidden layer between input and output layer.
In case of the $\mathrm{ReLU}$ as activation function, we can apply a symmetry procedure (Remark~\ref{rem:shallowdef} in Appendix~\ref{app:Dudley}) to represent them as sums
\begin{align}\label{def:Shallow}
\phi_{\bar{p}}(x)= \sum_{i=0}^n \phi_{p_i}(x) \, ,
\end{align}
where $\bar{p}$ denotes the tuple $(p_1,\dots,p_n)$.

\begin{assumption}\label{ass:parameter}
The parameters $\bar{p}$, which index shallow $\mathrm{ReLU}$ networks, are taken from a set
\begin{align*}
\bar{P}\subset \Big(\rr^d \times \rr \times \{\pm 1\}\Big)^{\times n} \, .
\end{align*}
For $\bar{P}$, we assume that there exist constants $c_w \geq 0$ and $c_b\in [1, 3]$, such, that for all parameter tuples $\bar{p}=\big((w_1,b_1,\kappa_1),\dots,(w_n,b_n,\kappa_n)\big) \in \bar{P}$, we have
\begin{align*}
\|w_i\| \leq  c_w \quad \text{and} \quad -c_b \leq \frac{b_i}{\|w_i\|} \leq \sqrt{\ln2} \, .
\end{align*}
\end{assumption}
We denote the set of shallow networks indexed by a parameter set $\bar{P}$ by
\begin{align} \label{def:Phi}
\Phi_{\bar{P}}:= \big\{ \phi_{\bar{p}} \, : \, \bar{p} \in \bar{P} \big\} \, .
\end{align}
We now enrich the input space $\rr^d$ of the networks with a probability distribution, which reflects the sampling procedure and renders each neural network to a random variable.
Furthermore, a random label $y$ takes its values in the output space $\rr$, for which we assume the following.

\begin{assumption}\label{ass:statistic}
The random sample $x\in\rr^d$ and label $y\in\rr$ follow a joint distribution $\mu$, such that the marginal distribution $\mu_x$ of the sample $x$ is standard Gaussian with the density
\begin{align*}
\frac{1}{(2\pi)^{d/2}} \exp \left[ -\frac{\|x\|^2}{2}\right] \, .
\end{align*}
As available data, we assume independent copies $\{(x_j,y_j)\}_{j=1}^m$ of the random pair $(x,y)$, each distributed by $\mu$.
\end{assumption}

\section{Concentration of the Empirical Norm}\label{sec:empiric}

Supervised learning algorithms interpolate labels $y\in \mathcal{Y}$ of samples $x \in \mathcal{X}$, which are jointly distributed by $\mu$ on $\mathcal{X}\times \mathcal{Y}$.
This task often has to be solved under limited data accessibility.
The data, which is available for training, consist with Assumption~\ref{ass:statistic} of $m$ independent copies of the random pair $(x,y)$.
During training, the interpolation quality of a hypothesis function $f:\mathcal{X} \rightarrow \mathcal{Y}$ can only be evaluated at the given random samples $\{x_j\}_{j=1}^m\subset\mathcal{X}$.
Any algorithm therefore accesses each function $f$ through its \emph{sketch}
\begin{align*}
S[f] = \Big(f(x_1),\dots,f(x_m)\Big)^T \, ,
\end{align*}
where we refer to $S$ as the \emph{sample operator}.
After training, the quality of a resulting model is often measured by its generalization to new data, which was not employed in the training.
With $\rr^d \times \rr$ as the space $\mathcal{X}\times \mathcal{Y}$, we quantify the generalization error of a function $f$ by its \emph{expected risk}
\begin{align*}
\Big\|f- y \Big\|_{\mu}:= \sqrt{\mathbb{E}_{\mu} \Big(f(x) - y \Big)^2} \, .
\end{align*}
The functional $\|\cdot\|_{\mu}$ further provides the norm of the space $L^2(\rr^d, \mu_x)$, which consists of functions $f:\rr^d \rightarrow \rr$ with
\begin{align*}
\big\|f \big\|_{\mu}:= \sqrt{\mathbb{E}_{\mu_x} \Big(f(x)\Big)^2} \, .
\end{align*}
If the label $y$ depends deterministically on the associated sample $x$, we can treat $y$ as an element of $L^2(\rr^d, \mu_x)$ and the expected risk of any function $f$ is its function distance to $y$.
Sketching any hypothesis function $f$ with the sample operator $S$, we perform a Monte-Carlo approximation of the expected risk, which is called the \emph{empirical risk}
\begin{align*}
\Big\|f - y \Big\|_m := \frac{1}{\sqrt{m}} \Big\|S[f] - (y_1,...,y_m)^T\Big\|_2= \sqrt{\frac{1}{m} \sum_{j=1}^m \Big(f(x_j) - y_j \Big)^2} \, .
\end{align*}
The random functional $\|\cdot\|_m$ further defines a seminorm on $L^2(\rr^d, \mu_x)$, which we call the \emph{empirical norm} (see Definition~\ref{def:empiricalseminorm} in Appendix~\ref{app:NeuRIP}).
In Remark~\ref{rem:seminorm} we argue, that, under mild assumptions, $\|\cdot \|_m$ fails to be a norm.

In order to obtain a well generalizing model, one aims to identify a function $f$ with a low expected risk.
However, in the case of limited data, one is restricted to the optimization of the empirical risk.
Our approach to derive generalization guarantees is based on the stochastic relation of both risks.
If $\{x_j\}_{j=1}^m$ are independently distributed by $\mu_x$, the law of large numbers \cite{boucheron_concentration_2013} implies in the limit $m \rightarrow \infty$ for any $f\in L^2(\rr^d,\mu_x)$ the convergence
\begin{align*}
\|f\|_m \rightarrow \|f\|_{\mu} \, .
\end{align*}
While this states the asymptotic concentration of the empirical norm at the function norm for a single function $f$, we have to consider two issues to formulate our notion of norm concentration:
Firstly, we derive non-asymptotic results, that is bounds on the distances $\|f\|_m -\|f\|_{\mu}$ for a fixed number $m$ of samples.
Secondly, the bounds on the distance have to be satisfied uniformly for all functions $f$ in a given set.

Sample operators, which admit uniform concentration properties, have been studied as \emph{restricted isometries} in the area of compressed sensing \cite{eldar_compressed_2012}.
For shallow $\mathrm{ReLU}$ networks of the form \eqref{def:Shallow}, we define the restricted isometry property of the sampling operator $S$ as follows.
\begin{definition}\label{def:neurip}
Let $s\in(0,1)$ be a constant and $\bar{P}\subset \left(\rr^d \times \rr \times \{\pm 1\}\right)^{\times n}$ a parameter set.
We say, that the \textbf{Neural Restricted Isometry Property} $\big(\nrip(\bar{P})\big)$ is satisfied, if for all $\bar{p}\in \bar{P}$ it holds
\begin{align*}
(1- s) \| \phi_{\bar{p}} \|^2_{\mu} \leq \| \phi_{\bar{p}} \|_{m}^2 \leq (1+ s) \|\phi_{\bar{p}} \|^2_{\mu} \, .
\end{align*}
\end{definition}
In the following Theorem, we provide a bound on the number $m$ of samples, which is sufficient for the operator $S$ to satisfy $\nrip(\bar{P})$.
We postpone its proof to Section~\ref{sec:size}, where we introduce the key techniques to derive non-asymptotic uniform concentration statements.
\begin{theorem}\label{thm:NeuRIP}
There exist universal constants $C_1,C_2 \in \rr$, such that the following holds for a sample operator $S$, which is constructed from random samples $\{x_j\}_{j=1}^m$ respecting Assumption~\ref{ass:statistic}:
Let $\bar{P} \subset \left(\rr^d \times \rr \times \{\pm 1\}\right)^{\times n}$ be any parameter set satisfying Assumption~\ref{ass:parameter} and $\|\phi_{\bar{p}}\|_{\mu} \geq 1$ for all $\bar{p}\in\bar{P}$.
Then, for each $u \geq 2$ and $s \in (0,1)$, $\mathrm{NeuRIP_s(\bar{P})}$ is satisfied with probability at least $1-17\exp\left[-\frac{u}{4} \right]$, provided that
\begin{align*}
m \geq  n^3 c_w^2 \left( 8c_b+d + \frac{\ln2}{4}\right) \max \left( C_1 \frac{u}{s} \, , \, C_2n^2c_w^2 \left(\frac{u}{s}\right)^2  \right) \, .
\end{align*}
\end{theorem}

One should notice, that in Theorem~\ref{thm:NeuRIP} we have a tradeoff between the parameter $s$, which limits the deviation of $\|\cdot\|_m$ from $\|\cdot\|_{\mu}$, and the confidence parameter $u$.
Understanding the quotient $\frac{u}{s}$ as a precision parameter of the statement, the lower bound on the corresponding sample size $m$ is split into two scaling regimes.
While in the regime of low deviations and high probabilities the sample size $m$ has to scale quadratically with $\frac{u}{s}$ to satisfy the stated bound, in the regime of less precise statements one observes a linear scaling.

\section{Uniform Generalization of Sublevel Sets of the Empirical Risk}\label{sec:generalization}

In case of the $\nrip$ event, the function norm $\|\cdot\|_{\mu}$ corresponding to the expected risk is close to its empirical counterpart $\|\cdot\|_m$, which corresponds to the empirical risk.
Motivated by this property, we aim to find a shallow $\mathrm{ReLU}$ network $\phi_{\bar{p}}$ with small expected risk $\big\| \phi_{\bar{p}}- y \big\|_{\mu}$ by solving the \emph{empirical risk minimization} problem
\begin{align}\tag{$\mathrm{P}_{m,y}$} \label{def:Prym}
\min_{\bar{p} \in \bar{P}} \big\| \phi_{\bar{p}}- y \big\|_{m}  \, .
\end{align}
However, since the set $\Phi_{\bar{P}}$ of shallow $\mathrm{ReLU}$ networks is non-convex, \eqref{def:Prym} cannot be solved by efficient convex optimizers \cite{plan_generalized_2016,genzel_mismatch_2019}.
Instead of providing a generalization analysis only of the solution of \eqref{def:Prym}, we thus introduce a tolerance $\xi \geq 0$ for the empirical risk and provide bounds on the generalization error, which hold uniformly on the \emph{sublevel set}
\begin{align}\label{def:sublevel}
\bar{Q}_{y,\xi}:= \Big\{ \bar{q} \in \bar{P} \, : \, \big\| \phi_{\bar{p}}- y \big\|_{m} \leq \xi \Big\} \, .
\end{align}
Before discussing generic regression problems, we for now assume the label $y$ to be a neural network, which is parameterized by a tuple $\bar{p}^*$ in the hypothesis set $\bar{P}$.
For all $(x,y)$ in the support of $\mu$ we then have $y=\phi_{\bar{p}^*}(x)$ and the minimum of the expected risk on $\bar{P}$ is zero.
By applying the sufficient condition for $\nrip$ from Theorem~\ref{thm:NeuRIP} we can in this case state generalization bounds on $\bar{Q}_{y,\xi}$ for arbitrary $\xi \geq 0$.

\begin{theorem}\label{thm:Prpm}
Let $\bar{P}$ be a parameter set satisfying Assumption \ref{ass:parameter}, and let $u\geq2$ and $t >\xi\geq0$ be constants. Further, let the number $m$ of samples satisfy
\begin{align*}
m \geq 8 n^3 c_w^2 \left( 8c_b+d + \frac{\ln2}{4}\right) \max \left( C_1 \frac{u}{\left(t^2-\xi^2 \right)}  \, , \, C_2n^2c_w^2 \left( \frac{ u} { \left(t^2-\xi^2 \right)} \right)^2  \right) \,,
\end{align*}
where $C_1$ and $C_2$ are universal constants.
Let $\{(x_j,y_j)_{j=1}^m\}$ be a data set respecting Assumption~\ref{ass:statistic} and let there exist $\bar{p}^*\in\bar{P}$, such that $y_j=\phi_{\bar{p}^*}(x_j)$ holds for all $j\in[m]$.
Then, with probability at least $1-17\exp\left[ - \frac{u}{4} \right]$, we have for all $\bar{q}\in\bar{Q}_{y,\xi}$ the bound $\| \phi_{\bar{q}} - \phi_{\bar{p}^*}\|_{\mu} \leq t$.
\end{theorem}
\begin{proof}
We notice first that $\Phi_{\bar{P}}-\phi_{\bar{p}^*}$ is a set of shallow neural networks with $2n$ neurons.
We normalize the set of such networks with a function norm greater than $t$ and parameterize them by
\begin{align*}
\bar{R}_t := \Big\{ \bar{r}:= \frac{(\bar{p},-\bar{p}^*)}{\|\phi_{\bar{p}}  - \phi_{\bar{p}^*} \|_{\mu} } \, : \, \bar{p}\in \bar{P} \, ,  \, \|\phi_{\bar{p}}  - \phi_{\bar{p}^*} \|_{\mu}  > t \Big\} \, .
\end{align*}
We assume next that $\nrip(\bar{R}_t)$ holds for $s=1-\frac{\xi^2}{t^2}$.
In this case, for all $\bar{q}\in\bar{P}$ with $\|\phi_{\bar{q}} - \phi_{\bar{p}^*}\|_{\mu}>  t$, we have that $\|\phi_{\bar{q}} - \phi_{\bar{p}^*}\|_{m} \geq  \frac{\xi}{t} \|\phi_{\bar{q}} - \phi_{\bar{p}^*}\|_{\mu} > \xi$ and thus $\bar{q}\notin \bar{Q}_{\phi_{\bar{p}^*},\xi}$.
If $\nrip(\bar{R}_t)$ holds, $\bar{q}\in \bar{Q}_{\phi_{\bar{p}^*},\xi}$ therefore implies $\|\phi_{\bar{q}} - \phi_{\bar{p}^*}\|_{\mu}\leq t$.

We further notice, that $\bar{R}_t$ satisfies Assumption~\ref{ass:parameter} with a by $t^{-1}$ rescaled constant $c_w$ and normalization invariant $c_b$, if $\bar{P}$ satisfies it for a $c_w$ and $c_b$.
Theorem~\ref{thm:NeuRIP} provides the lower bound on the sample complexity of $\nrip(\bar{R}_t)$ and finishes the proof.
\end{proof}
\begin{figure}[t]
\centering
\begin{tikzpicture}
	\draw[thick,->] (0,0) coordinate (or)  -- (11,0) coordinate node[right] {$\|\phi_{\bar{p}} - \phi_{\bar{p}^*} \|_{\mu} $};
	\draw[thick,->] (0,0) coordinate (or)  -- (0,3.5) coordinate (ha2);
	\draw[dashed] (0,0) coordinate (or) node[below] {$ 0 $}  -- (9,3) coordinate (h1) node[right=-0.1] {$ \sqrt{1-s} \, \|\phi_{\bar{p}} - \phi_{\bar{p}^*} \|_{\mu} $};
	\draw[dashed] (0,0) coordinate (or) node[below] {$ 0 $}  -- (4,3) coordinate (h1) node[left=0.1] {$ \sqrt{1+s} \, \|\phi_{\bar{p}} - \phi_{\bar{p}^*} \|_{\mu}$};
	\draw[dashed] (3,0) coordinate (or) node[below] {$ t $} -- (3,2.25) coordinate (h2) ;
	\fill[fill=gray!20] (3,1)--(9,3)--(3,2.25);
	\fill[fill=gray!20] (4,3)--(9,3)--(3,2.25);
	\draw (4,2.5) node [right] {$ \|\phi_{\bar{p}}- \phi_{\bar{p}^*} \|_m$};
	\draw[out=10, in=-180] (0,0) to (1.3, 0.2);
	\draw[out=0, in=-180] (1.3,0.2) to (2, 0);
	\draw[out=0, in=-180] (2,0) to (3.3, 2);
	\draw[out=0, in=-160] (3.3,2) to (8, 3);
\end{tikzpicture}
\caption{Sketch of the empirical risk $\big\|\phi_{\bar{p}} - \phi_{\bar{p}^*} \big\|_m$, which is the objective of problem \eqref{def:Prym} in case of $y=\phi_{\bar{p}^*}$.
If the event $\nrip(\bar{R}_t)$ holds, the empirical risk of any network $\phi_{\bar{p}}$ satisfies the bounds sketched by the grey area, if the corresponding expected risk exceeds $t$.
Vice versa, the sublevel set $\bar{Q}_{y,\xi}$ of the empirical risk at the level $\xi=\sqrt{1-s}\, t$ is bounded by $t$ in the expected risk, if $\nrip(\bar{R}_t)$ holds.}
\label{fig:losssketch}
\end{figure}
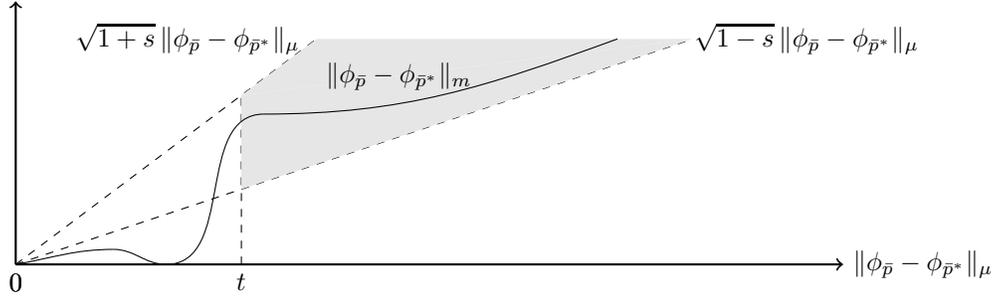

At any network, where an optimization method might terminate, the concentration of the empirical at the expected risk can be achieved with less data than required to achieve an analogous $\nrip$ event.
However, in the chosen stochastic setting, the termination of an optimization flow and the norm concentration at that network cannot be assumed to be independent events.
We overcome this problem by not specifying the outcome of an optimization method and instead state uniform bounds on the norm concentration.
The only assumption on an algorithm to state the generalization bound is then the identification of a network, which allows an upper bound $\xi$ on its empirical risk.
As sketched in Figure~\ref{fig:losssketch}, the event $\nrip(\bar{R}_t)$ then restricts the expected risk to be below the corresponding level $t$.

We continue to discuss the empirical risk surface for generic distributions $\mu$ satisfying Assumption~\ref{ass:statistic}, where $y$ is not necessarily a neural network.

\begin{theorem}\label{the:Prym}
There are constants $C_0,C_1,C_2,C_3,C_4$ and $C_5$, such that the following holds: Let $\bar{P}$ satisfy Assumption~\ref{ass:parameter} for some constants $c_w, c_b$ and let $\bar{p}^*\in \bar{P}$ be such that, for some $c_{\bar{p}^*}\geq 0$, we have
\begin{align*}
\mathbb{E}_{\mu}\left[ \exp\left( \frac{(\phi_{\bar{p}^*}(x)-y)^2}{c_{\bar{p}^*}^2} \right) \right] \leq 2 \, .
\end{align*}
We assume for a given $s\in (0,1)$ and confidence parameter $u>0$, that the number $m$ of samples is large enough such that
\begin{align} \label{def:alpha}
\alpha:=\frac{m}{n^3 c_w^2 ( 8 c_b +d + \frac{\ln2}{4})}  \geq 8 \max \left( C_1 \frac{(1-s)^2 u}{s \,\|\phi_{\bar{p}^*}-y \|_{\mu}^2} \, , \, C_2 n^2c_w^2  \left(\frac{ (1-s)^2 u}{s\,\|\phi_{\bar{p}^*} -y\|_{\mu}^2} \right)^2  c_{\bar{p}^*}^2 \right) \, .
\end{align}
We further choose confidence parameters $v_1,v_2>C_0$ and define for some $\omega \geq 0$ the parameter
\begin{align*} 
\eta := \left( \frac{2}{(1-s)} + 1 \right) \|\phi_{\bar{p}^*}-y\|_{\mu} +  \sqrt{\frac{C_3 v_1v_2c_{\bar{p}^*}}{1-s}} \, \alpha^{-\frac{1}{4}} + \frac{\omega}{\sqrt{1-s}} \, .
\end{align*}
If we set $\xi=\sqrt{\| \phi_{\bar{p}^*} - y\|^2_m + \omega^2}$ as the tolerance for the empirical risk, then the probability, that all $\bar{q}\in \bar{Q}_{y,\xi}$ satisfy $\|\phi_{\bar{q}} -y\|_{\mu} \leq \eta$, is at least
\begin{align*}
1- 2\exp\left[ -C_4 m v_1^2\right]  - 2\exp\left[ -C_5 v_2^2\right] - 17 \exp\left[ - \frac{u}{4} \right] \, .
\end{align*}
\end{theorem}
\begin{proof}[Proof sketch of Theorem~\ref{the:Prym} (complete proof in Appendix~\ref{app:Agnostic} as Corollary~\ref{cor:agnostic})]
We first define and decompose the \emph{excess risk} by
\begin{align} 
\mathcal{E}(\bar{q},{\bar{p}^*}):=& \|\phi_{\bar{q}} - y\|^2_m-\|\phi_{\bar{p}^*} - y\|^2_m \nonumber \\
=& \|\phi_{\bar{q}} - \phi_{\bar{p}^*}\|^2_m + \frac{2}{m} \sum_{j=1}^m \left(\phi_{\bar{p}^*}(x_j)-y_j \right) \left( \phi_{\bar{q}}(x_j)- \phi_{\bar{p}^*}(x_j)\right) \, . \label{eq:excessdecom}
\end{align}
It suffices to show, that within the stated confidence level we have $\mathcal{E}({\bar{q}},{\bar{p}^*})> \omega^2$ for all $\bar{q}\in \bar{P}$ with $\|\phi_{\bar{q}} - \phi_{\bar{p}^*}\|_{\mu} > \eta - \|\phi_{\bar{p}^*}-y\|_{\mu}$.
We notice that this implies the claim, since $\bar{q}\in \bar{Q}_{y,\xi}$ is equivalent to $\mathcal{E}({\bar{q}},{\bar{p}^*})\leq \omega^2$, which then implies
\begin{align*} 
\|\phi_{\bar{q}}-y\|_{\mu} \leq \|\phi_{\bar{q}} - \phi_{\bar{p}^*}\|_{\mu} +\|\phi_{\bar{p}^*}-y\|_{\mu}\leq\eta \, .
\end{align*}
The expectation of the first term in the decomposed excess risk \eqref{eq:excessdecom} is $\|\phi_{\bar{q}} - \phi_{\bar{p}^*}\|_{\mu}^2$, and the expectation of the second term can be lower bounded with use of the Cauchy-Schwarz Inequality, which yields the bound
\begin{align*} 
\mathbb{E}_{\mu}\left[\frac{2}{m} \sum_{j=1}^m \left(\phi_{\bar{p}^*}(x_j)-y_j\right) \left( \phi_{\bar{q}}(x_j)- \phi_{\bar{p}^*}(x_j)\right) \right] \geq - 2\|\phi_{\bar{p}^*}-y \|_{\mu} \|\phi_{\bar{q}} - \phi_{\bar{p}^*}\|_{\mu} \, .
\end{align*}
Hence, for $\eta > 3 \|\phi_{\bar{p}^*}-y \|_{\mu} $, we have $\mathbb{E}_{\mu}\left[\mathcal{E}({\bar{q}},{\bar{p}^*})\right]> 0$.
We are thus left to strengthen the condition on $\eta$ to achieve $\mathbb{E}_{\mu}\left[\mathcal{E}({\bar{q}},{\bar{p}^*})\right]> \omega^2$ and control the fluctuation of $\mathcal{E}({\bar{q}},{\bar{p}^*})$ uniformly around its expectation.
To achieve a uniform bound on the fluctuation of the first term, we apply Theorem~\ref{thm:NeuRIP}.
The concentration rate of the second term is provided by Lemma~\ref{lem:neuralmultiplier} in Appendix~\ref{app:Agnostic} and is proven similary to Theorem~\ref{thm:NeuRIP} with chaining techniques, which are discussed in Section~\ref{sec:size}.
In Appendix~\ref{app:Agnostic} we then provide with Theorem~\ref{the:empiricalnoise} a general bounds to achieve $\mathcal{E}({\bar{q}},{\bar{p}^*})> \omega^2$ uniformly for all $\bar{q}$ with $\|\phi_{\bar{q}} - \phi_{\bar{p}^*}\|_{\mu} > \eta - \|\phi_{\bar{p}^*}-y\|_{\mu}$, from which Theorem~\ref{the:Prym} follows as a simplification (Corollary~\ref{cor:agnostic}).
\end{proof}

We notice in Theorem~\ref{the:Prym}, that in the limit of infinite data $m$ one can choose an asymptotically small deviation constant $s$ and the derived bound $\eta$ on the generalization error converges to $3\|\phi_{\bar{p}^*}-y\|_{\mu}+\omega$.
This reflects a lower limit of the generalization bound, which is a sum of the theoretically achievable minimum of the expected risk and the additional tolerance $\omega$, up to which \eqref{def:Prym} is assumed to be solved by realistic optimization algorithms.

\section{Size Control of Stochastic Processes on Shallow Networks}\label{sec:size}

We now introduce the key techniques to derive concentration statements for the empirical norm, which hold uniformly on sets of shallow $\mathrm{ReLU}$ networks. 
First of all, we rewrite the event $\nrip(\bar{P})$ by treating the norm difference $\|\phi_{\bar{p}}\|^2_m - \|\phi_{\bar{p}} \|_{\mu}^2 $ as a stochastic process, which is indexed by a parameter set $\bar{P}$.
The event $\nrip(\bar{P})$ holds, if and only if we have
\begin{align}\label{def:size}
s\geq \sup_{\bar{p}\in \bar{P}} \frac{\Big| \| \phi_{\bar{p}} \|^2_m - \|\phi_{\bar{p}} \|_{\mu}^2 \Big|}{\|\phi_{\bar{p}}\|^2_{\mu}} \, .
\end{align}
The supremum of stochastic processes has been studied as their \emph{size} \cite{talagrand_upper_2014}.
To bound the size of a process, one has to understand the correlation of its variables.
To this end, we define the \emph{Sub-Gaussian metric} for any two parameter tuples $\bar{p},\bar{q}\in \bar{P}$ as
\begin{align*}
d_{\psi_2}( \phi_{\bar{p}}, \phi_{\bar{q}}) := \inf \Big\{ C_{\psi_2} \geq 0 \, : \, \mathbb{E} \left[\exp\left( \frac{|\phi_{\bar{p}}(x)- \phi_{\bar{q}}(x)  |^2}{C_{\psi_2}^2} \right) \right]\leq 2 \Big\} \, .
\end{align*}
A small Sub-Gaussian metric between random variables implies, that their values are likely to be close.
To capture the Sub-Gaussian structure of a process, we introduce $\epsilon$-\emph{nets} in the Sub-Gaussian metric, which are for an $\epsilon>0$ subsets ${\bar{Q}} \subset {\bar{P}}$ such that, for any $\bar{p}\in\bar{P}$, there exists $\bar{q}\in\bar{Q}$ satisfying
\begin{align*}
d_{\psi_2}( \phi_{\bar{p}}, \phi_{\bar{q}}) \leq \epsilon \, .
\end{align*}
The smallest cardinality of an $\epsilon$-net $\Phi_{\bar{Q}}$ is called the \emph{Sub-Gaussian covering number} $\mathcal{N}(\Phi_{\bar{P}},d_{\psi_2},\epsilon)$.
The next Lemma provides a bound for such covering numbers in the situation of shallow $\mathrm{ReLU}$ networks.
\begin{lemma}\label{lem:subgaussiancovering}
Let $\bar{P}$ be a parameter set satisfying Assumption~\ref{ass:parameter}. Then, there exists a set $\hat{P}$ with $\bar{P}\subset \hat{P}$ and
\begin{align} \label{eq:card}
\mathcal{N}(\Phi_{\hat{P}},d_{\psi_2},\epsilon) \leq  2^n \cdot \Big\lfloor \frac{16nc_bc_w}{\epsilon} +1 \Big\rfloor^n \cdot  \Big\lfloor \frac{32nc_bc_w}{\epsilon}  +1 \Big\rfloor^n  \cdot \left(1+\frac{1}{\sin(\frac{\epsilon}{16nc_w})}\right)^{nd} \, .
\end{align}
\end{lemma}
\begin{proof}[Proof sketch of Lemma~\ref{lem:subgaussiancovering} (complete proof in Appendix~\ref{app:Dudley} as Theorem~\ref{the:coveringshallow})]
We first restrict to the case $n=1$ and introduce in Appendix~\ref{app:Approximation} a multi-resolutional approximation scheme for $\mathrm{ReLU}$ neurons.
The approximation scheme consists in the independent discretization of the direction and norm of the weight and the quotient of the bias with the norm of the weight.
We bound the approximation error in the Sub-Gaussian metric in Appendix~\ref{app:Covering} and apply our findings in the construction of $\epsilon$-nets of the $\mathrm{ReLU}$ neurons.
By covering each neuron with an $\frac{\epsilon}{n}$-net we then construct for an arbitrary number $n$ of neurons an $\epsilon$-nets, which cardinality is equal to the right hand side of \eqref{eq:card}.
\end{proof}

To provide bounds of the form \eqref{def:size} on the size of a process, we apply the \emph{generic chaining} method \cite{talagrand_upper_2014}.
This method provides bounds in terms of the \emph{Talagrand-functional} of the process in the Sub-Gaussian metric, which we define in the following.

\begin{definition}\label{def:Talagrand}
Let $(T,d)$ be a metric space. We say a sequence $(T_k)_{k=0}^{\infty}$ of subsets $T_k \subset T$ is \textbf{admissible}, if
\begin{align*}
\big|T_k\big| \leq 2^{(2^k)} \quad \text{and} \quad \big|T_0\big|=1 \, .
\end{align*}
The \textbf{Talagrand-functional} of the metric space is then defined as
\begin{align*}
\gamma_2(T,d) := \inf\limits_{(T_k)} \sup\limits_{t \in T} \sum_{k=0}^{\infty} 2^{\frac{k}{2}} d(t,T_k) \, ,
\end{align*}
where the infimum is taken over all admissible sequences.
\end{definition}

With the bounds on the Sub-Gaussian covering number, which are provided by Lemma~\ref{lem:subgaussiancovering}, we bound the Talagrand-functional for shallow $\mathrm{ReLU}$ networks in the following Lemma.

\begin{lemma} \label{lem:dudleyshallow}
Let $\bar{P}$ satisfy Assumption~\ref{ass:parameter}. Then we have
\begin{align*}
\gamma_2(\Phi_{\bar{P}},d_{\psi_2}) \leq \frac{ \sqrt{2}}{(\sqrt{2}-1)\sqrt{\ln2}} \int_{0}^{\infty}  \sqrt{\ln \mathcal{N}\left(\Phi_{\bar{P}},d_{\psi_2}, {\epsilon} \right)} \, d \epsilon \leq  \frac{8}{(2-\sqrt{2})\sqrt{\ln2}} n^{\frac{3}{2}} c_w  \sqrt{8c_b + d + \frac{\ln2}{4}} \, .
\end{align*}
\end{lemma}
\begin{proof}[Proof sketch of Lemma~\ref{lem:dudleyshallow} (complete proof in Appendix~\ref{app:Dudley} as Theorem~\ref{the:dudleyshallow2})]
In order to apply the covering number bounds in Lemma~\ref{lem:subgaussiancovering}, we enlarge $\bar{P}$ to $\hat{P}$ and notice that $\gamma_2(\Phi_{\bar{P}},d_{\psi_2}) \leq \gamma_2(\Phi_{\hat{P}},d_{\psi_2})$.
We then apply Dudleys entropy bound (Lemma~\ref{lem:dudleygeneral} in Appendix~\ref{app:Dudley}) on the functional $\gamma_2(\Phi_{\hat{P}},d_{\psi_2})$, which yields the estimate
\begin{align}\label{est:Dudley}
\gamma_2(\Phi_{\bar{P}},d_{\psi_2}) \leq N(\Phi_{\bar{P}}):= \frac{\sqrt{2}}{(\sqrt{2}-1)\sqrt{\ln2}}  \int_{0}^{\infty} \sqrt{\ln \mathcal{N}(\Phi_{\hat{P}},d_{\psi_2},\epsilon)} \, d\epsilon \, .
\end{align}
With the bound $2c_w$ on the Sub-Gaussian norm of each neuron, which we provide with Theorem~\ref{the:radius} in Appendix~\ref{app:Covering}, the Sub-Gaussian norm of each shallow network is bounded by $2nc_w$.
This implies $\mathcal{N}(\Phi_{\hat{P}},d_{\psi_2},\epsilon)=1$ for $\epsilon\geq 2nc_w$, which enables us to take finite integration bounds.
The integral on the right hand side of \eqref{est:Dudley} can then be estimated and we arrive at the stated bound after an application of the Cauchy-Schwarz inequality.
\end{proof}

In Appendix~\ref{app:Agnostic}, we derive with Theorem~\ref{the:mendelsonshallow} bounds on the $\Lambda$-functionals (see Definition~\ref{def:mendelson}), which have been introduced in \cite{mendelson_upper_2016} as generalizations of the Talagrand-functional.
In Lemma~\ref{lem:neuralmultiplier} we then apply these results to provide uniform bounds for the second term of the excess risk decomposition \eqref{eq:excessdecom}.
In the reminder of this section we focus on providing the bound \eqref{def:size} and state the following Lemma, which we will prove in Appendix~\ref{app:NeuRIP} as Lemma~\ref{lem:supquadratical2}.
\begin{lemma}\label{lem:supquadratical}
Let $\Phi_{\bar{P}}$ be any set of real functions indexed by a parameter set $\bar{P}$ and define
\begin{align*}
N(\Phi_{\bar{P}}) := \frac{ \sqrt{2}}{(\sqrt{2}-1)\sqrt{\ln2}} \int_{0}^{\infty}  \sqrt{\ln \mathcal{N}\left(\Phi_{\bar{P}},d_{\psi_2}, {\epsilon} \right)} \, d \epsilon  \quad \text{and} \quad \Delta(\Phi_{\bar{P}}) := \sup_{\bar{p}\in\bar{P}} \|\phi_{\bar{p}}\|_{\psi_2} \, .
\end{align*}
Then, for any $u\geq 2$, we have with probability at least $1 - 17 \exp\left[ - \frac{u}{4}\right]$ that
\begin{align*}
\sup_{\bar{p} \in \bar{P} } \Big| \|\phi_{\bar{p}} \|^2_m - \|\phi_{\bar{p}}\|^2 \Big| \leq  \frac{u}{\sqrt{m}} \left[ 25 \frac{N(\Phi_{\bar{P}})}{m^{\frac{1}{4}}}   +   \sqrt{85 \Delta(\Phi_{\bar{P}}) \, N(\Phi_{\bar{P}})}\right]^2  \, .
\end{align*}
\end{lemma}

The bounds on the sample complexity for achieving the $\nrip$ event, which are provided in Theorem~\ref{thm:NeuRIP}, are now proven with application of the above Lemmata.

\begin{proof}[Proof of Theorem~\ref{thm:NeuRIP}]
Since we assume $\|\phi_{\bar{p}}\|_{\mu}\geq 1$ for $\bar{p}\in\bar{P}$, we have
\begin{align}\label{sizeboundneurip}
\tilde{s}:=\sup_{\bar{p}\in \bar{P}} \frac{\Big| \| \phi_{\bar{p}} \|^2_m - \|\phi_{\bar{p}} \|_{\mu}^2 \Big|}{\|\phi_{\bar{p}}\|^2_{\mu}} \leq \sup_{\bar{p}\in \bar{P}}  \Big| \| \phi_{\bar{p}} \|^2_m - \|\phi_{\bar{p}} \|_{\mu}^2 \Big| \, .
\end{align}
By applying Lemma~\ref{lem:supquadratical} we bound the right hand side of \eqref{sizeboundneurip} and further apply Dudleys entropy bound on shallow $\mathrm{ReLU}$ networks, which was stated with Lemma~\ref{lem:dudleyshallow}.
The $\nrip(\bar{P})$ event holds in case of $s\geq \tilde{s}$, and the sample complexities provided in Theorem~\ref{thm:NeuRIP} follow from a refinement of this condition (for details see the proof of Theorem~\ref{thm:NeuRIP2} in Appendix~\ref{app:NeuRIP}).
\end{proof}

In Appendix~\ref{app:NeuRIP} we provide with Theorem~\ref{thm:NeuRIP2} a more general version of Theorem~\ref{thm:NeuRIP}, where the assumption of a uniform lower bound of the network norm $\|\phi_{\bar{p}}\|_{\mu}$ can be weakened.

\section{Conclusion and Outlook} \label{sec:conclusion}

In this work, we have investigated the empirical risk surface of shallow $\mathrm{ReLU}$ networks in terms of uniform concentration events of the empirical norm.
More precisely, we have defined the Neural Restricted Isometry Property ($\nrip$) and bounded the sample complexity to achieve $\nrip$ (Theorem~\ref{thm:NeuRIP}), which depend on realistic parameter bounds and the network architecture.
We applied our findings to derive upper bounds on the expected risk, which hold uniformly in sublevel sets of the empirical risk.
Provided that a network optimization algorithm can identify a network with a small empirical risk, the identified network is with our results guaranteed to generalize.
By deriving uniform concentration statements, we have overcome the problem, that the termination of an optimization algorithm at a network and the empirical risk concentration at this network are not independent events.
However, the set of networks, where descent algorithms are assumed to terminate, can be further narrowed down to consist of the local minima in the empirical risk surface \cite{lee_gradient_2016}.
Evidence has been derived that such minima are found in lower dimensional subsets within the here discussed sublevel sets of the empirical risk \cite{choromanska_loss_2015,venturi_spurious_2019}.
By providing uniform bounds on the entire sublevel set of the empirical risk, we thus have overestimated the set of possible termination points of common learning algorithms.
In future work, we aim to perform the uniform empirical norm concentration on the critical points of the empirical risk instead, which we expect to allow even sharper bounds for the sample complexity.

Furthermore, we intend to apply our methods to more general input distributions than the assumed standard Gaussian distribution.
If generic Gaussian distributions can be treated in the case of shallow networks, one can then derive bounds for the Sub-Gaussian covering number for deep $\mathrm{ReLU}$ networks by induction through the layers.
We further expect bounds on the covering number to hold also for generic Lipschitz continuous activation functions different from the $\mathrm{ReLU}$.
Our intuition in this proposition is built on the concentration of measure phenomenon \cite{ledoux_concentration_2005}, which provides bounds on the Sub-Gaussian norm of functions on normal concentrating input spaces.
Since such bounds scale with the Lipschitz constant of the function, one can apply them to find $\epsilon$-nets for neurons, which have an identical activation pattern.
For a full analysis of the covering numbers, such bounds would have to be complemented with a discretization scheme of the activation pattern  \cite{montufar_number_2014-1}, similar to our scheme derived in Appendix~\ref{app:Approximation}.

\section*{Broader Impact}

At this date, supervised machine learning is affecting personal and public lives on a broad scale.
The generalization of empirically trained models is the central property to render them reliable and safe.
Our analysis aims at a profound understanding of the interplay of generalization, architectural choices, and available data.
We have provided a conceptual discussion and proved the effectiveness of applying uniform concentration events for generalization guarantees of common supervised machine learning algorithms.

\section*{Acknowledgement}
The authors are grateful to Philipp Trunschke, Qiao Luo, Martin Genzel and Reinhold Schneider for many fruitful discussions. The work was funded by the MATH+ Research Center through project EF1-4.
G.K. also acknowledges partial support by the Bundesministerium für Bildung und Forschung (BMBF) through the Berliner Zentrum for Machine Learning (BZML) and the Berlin Institute for the Foundations of Learning and Data (BIFOLD), Project AP4.

\bibliographystyle{alpha}

\newcommand{\etalchar}[1]{$^{#1}$}

\medskip

\small

\appendices


\section{Discretization of the Activation Pattern and Gradient of $\mathrm{ReLU}$ Neurons}\label{app:Approximation}

According to its definition in Section~\ref{sec:assumptions}, any $\mathrm{ReLU}$ neuron $\phi_{(w,b,\kappa)}$ is a piecewise linear function with linear regions bounded by the hyperplane
\begin{align*}
h_{w,b}=\{ x \in \rr^d \, : \, \braket{w,x}+b=0\} \, .
\end{align*}
By this hyperplane, the input space $\rr^d$ is split into two regions, which we call the \emph{activation pattern} (for a sketch see Figure~\ref{fig:Gradientapproximation}). 
We furthermore collect the weight $w$, the bias $b$ and the sign $\kappa$ in the parameter tuple $p=(w,b,\kappa)$.
Given a set $P\subset \rr^d \times \rr \times \{\pm1\}$ of parameter tuples we define a set of neurons by
\begin{align*}
\Phi_P=\{\phi_{p} \, : \, p=(w,b,\kappa) \in P \} \, .
\end{align*}
In the following, we will determine finite parameter sets $\tilde{P}$, which are rich enough to contain for each $p\in P$ a $\tilde{p}\in\tilde{P}$ with an approximatively same activation pattern and function gradient. 
As the first step towards this aim, we introduce the concept of angular covering sets. 
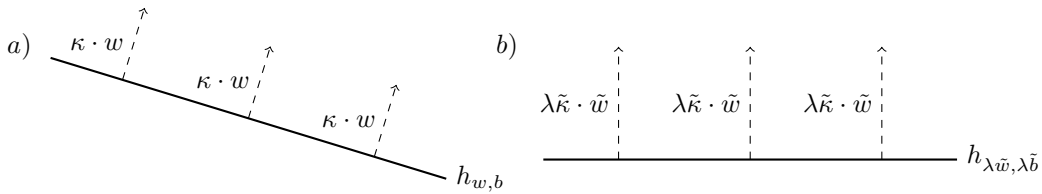
\begin{figure}[h]
\centering
\begin{tikzpicture}
	\draw[thick] (-5,2.5) coordinate (ha1)  -- (0.5,2.5) coordinate (ha2) node[right] {$h_{\lambda \tilde{w},\lambda \tilde{b} }$};
	\draw[->,dashed] (-4,2.5) -- (-4,4) node[midway,left] {$\lambda \tilde{\kappa} \cdot \tilde{w}$};
	\draw[->,dashed] (-2.25,2.5) -- (-2.25,4) node[midway,left] {$\lambda \tilde{\kappa} \cdot \tilde{w}$};
	\draw[->,dashed] (-0.5,2.5) -- (-0.5,4) node[midway,left] {$\lambda \tilde{\kappa} \cdot \tilde{w}$};
	\draw (-12.25,4) node [right] {$a)$};
	\draw (-5.75,4) node [right] {$b)$};
	\begin{scope}[shift={(-7.5,0)},rotate=-17]
	\draw[thick] (-5,2.5) coordinate (ha1)  -- (0.5,2.5) coordinate (ha2) node[right] {$h_{w,b}$};
	\draw[->,dashed] (-4,2.5) -- (-4,3.5) node[midway,left] {$\kappa \cdot w$};
	\draw[->,dashed] (-2.25,2.5) -- (-2.25,3.5) node[midway,left] {$\kappa \cdot w$};
	\draw[->,dashed] (-0.5,2.5) -- (-0.5,3.5) node[midway,left] {$\kappa \cdot w$};
	\end{scope}
	 \end{tikzpicture}
 \caption{ \textit{a) Sketch of a neuron with parameters $(w,b,\kappa)$, which determine the activation change hyperplane $h_{w,b}$ and the gradient $\kappa w$ of the neuron in the non-vanishing region.
We approximate the function by a neuron b) with parameters $(\lambda \tilde{w},\lambda\tilde{b},\tilde{\kappa})$.
To this end, we set $\tilde{\kappa}=\kappa$ and choose $\lambda$ (respectively $\tilde{b}$) to be close to $\|w\|$ (respectively $\frac{b}{\|w\|}$).
The normalized weight $\tilde{w}$ is furthermore taken from an angle covering of $\mathbb{S}^{d-1}$. }
 \label{fig:Gradientapproximation}}
\end{figure}
\begin{definition}\label{def:anglecoveringset}
Let $W\subset \rr^d$ and $\gamma>0$, we say $N_{\gamma}(W,\angle)\subset W$ is an \textbf{angle covering set} with distortion at most $\gamma$, if for all $w\in W$ we find $\tilde{w}\in N_{\gamma}(W) $ such that
\begin{align*}
\min_{\tilde{w} \in N_{\gamma}} \angle(w,\tilde{w}) \leq \gamma \quad \text{where} \quad \angle(w,\tilde{w}) := \cos^{-1}\left( \frac{|\braket{w,\tilde{w}}|}{\|w\| \|\tilde{w}\|} \right)  \,.
\end{align*}
If $w=0$ or $\tilde{w}=0$ we set $\angle(w,\tilde{w})=0$.
We call the minimum of the cardinality $|N_{\gamma}(W, \angle)|$ among all $N_{\gamma}(W, \angle)$ the \textbf{angle covering number} $\mathcal{N}(W,\angle,\gamma)$.
\end{definition}
We are in particular interested in the angular covering number of the sphere $\mathbb{S}^{d-1}$, on which we derive a bound in the following Lemma.
\begin{lemma}\label{lem:Anglecoveringnumber}
We have for $\gamma\in(0,\pi]$ that 
\begin{align} \label{eq:angleestimation}
\mathcal{N}(\mathbb{S}^{d-1},\angle,\gamma) \leq \left(1+\frac{1}{\sin(\frac{\gamma}{2})} \right)^d \, .
\end{align}
\end{lemma}
\begin{proof}
Let $\epsilon\in (0,2]$ and $N_{\epsilon}(\mathbb{S}^{d-1},\|\cdot\|)$ be an $\epsilon$-net, which covers the sphere $\mathbb{S}^{d-1}$ in the Euclidean metric $\|\cdot\|$ (see \cite[Definition 4.2.1]{vershynin_high-dimensional_2018}).
For each $w\in \mathbb{S}^{d-1}$ there exists an element $\tilde{w}\in N_{\epsilon}$ satisfying
\begin{align*}
\braket{w,\tilde{w}}= \frac{1}{2} \left( \|w\|^2+\| \tilde{w} \|^2 - \|w-\tilde{w} \|^2 \right) \geq 1- \frac{1}{2} \epsilon^2\, .
\end{align*}
This allows us to estimate the angle $ \angle(w,\tilde{w})$ as
\begin{align}\label{equ:anglelower}
 \angle(w,\tilde{w}) \leq \cos^{-1} \left( 1- \frac{\epsilon^2}{2} \right) \, .
\end{align}
We apply \cite[Corollary 4.2.13]{vershynin_high-dimensional_2018}, which states, that for each $\epsilon\in (0,2]$ there exists an $\epsilon$-net $N_{\epsilon}(\mathbb{S}^{d-1},\|\cdot\|)$ of $\mathbb{S}^{d-1}$ in the Euclidean norm with cardinality bounded from above by $(1+\frac{2}{\epsilon})^d$. 
For any $\gamma \in(0,\pi]$, we now choose $\epsilon = \sqrt{2-2\cos(\gamma)}$.
Estimation \eqref{equ:anglelower} on each pair $w,\tilde{w}$ then implies, that $N_{\epsilon}(\mathbb{S}^{d-1},\|\cdot\|)$ an angle covering set with distortion at most $\gamma$. 
This allows us to conclude
\begin{align*}
\mathcal{N}(\mathbb{S}^{d-1},\angle,\gamma) \leq \Big|N_{\sqrt{2-2\cos(\gamma)}}(\mathbb{S}^{d-1},\|\cdot\|)\Big| = \left(1+\frac{2}{\sqrt{2-2\cos(\gamma)}} \right)^d =  \left(1+\frac{1}{\sin(\frac{\gamma}{2})} \right)^d \, .
\end{align*}
\end{proof}

In the next Lemma we now provide an approximation scheme of $\mathrm{ReLU}$ neurons.
When approximating a neuron $p=(w,b,\kappa)$ by $\tilde{p}=(\tilde{w},\tilde{b},\tilde{\kappa})$, we define $\rho_w$ (respectively $\rho_{\tilde{w}}$) to be the distance of the intersections of the hyperplanes $h_{w,b}$ (respectively $h_{ \tilde{w},\tilde{b}}$) with the axes $\{ a \cdot \frac{w}{\|w\|}\, , \, a \in \rr \}$ and $\{ a \cdot \tilde{w}\, , \, a \in \rr \}$ (see Figure~\ref{fig:Hyperplanes}).

\begin{lemma}\label{lem:Approximationscheme}
Let $P\subset \rr^d \times \rr\times \{\pm 1\}$ be a parameter set and $c_{w},c_{b}\in \mathbb{N}$ be constants such that, for any $(w,b,\kappa)\in P$, it holds
\begin{align*}
{\|w\|} \leq c_{w} \quad \text{and} \quad \frac{|b|}{\|w\|} \leq c_{b}  \, .
\end{align*}
For $\delta,\rho >0$, we define 
\begin{align*}
c_{\delta} = \Big\lfloor \frac{c_{w}}{\delta} \Big\rfloor \quad \text{and} \quad  c_{\rho}= \Big\lfloor \frac{c_b}{\rho} \Big\rfloor \, .
\end{align*}
Let further $\gamma>0$ and $N_{\gamma}(\mathbb{S}^{d-1},\angle)$ be an angular covering set.
Then, for each neuron $\phi_{(w,b,\kappa)}$ with parameters $(w,b,\kappa) \in P$, there exists another neuron $\phi_{(\lambda \tilde{w}, \lambda \tilde{b},\kappa)}$ with 
\begin{align}\label{parameterchoices}
\tilde{w} \in N_{\gamma}(\mathbb{S}^{d-1},\angle) \, \text{,} \quad
\tilde{b} \in \{-c_{\rho} \rho,(-c_{\rho}+1)\rho,...,c_{\rho} \rho\} \, \text{,} \quad
\tilde{\kappa}=\kappa \quad \text{and} \quad
\lambda \in \{ 0,\delta,2\delta, \dots , c_{\delta} \delta \} \, ,
\end{align}
such that 
\begin{align*}
\beta:=\angle(w,\tilde{w}) \leq \gamma &\text{,} \quad \|w-\lambda\tilde{w}\|^2\leq  {\delta^2} + 2(1-\cos(\beta))c_w^2, \quad  \max({\rho}_{w},{\rho}_{\tilde{w}})  \leq \rho \left[1 + \left( \frac{1}{\cos(\beta)}-1\right) c_{\rho}\right]\\
&\text{and}  \quad \big| \phi_{(w,b,\kappa)}(0)-\phi_{(\lambda \tilde{w}, \lambda \tilde{b},\tilde{\kappa})}(0) \big| \leq  \delta\rho ({c_{\delta}+c_{\rho}})\, .
\end{align*} 
\end{lemma}

\begin{figure}[t]
\centering
\begin{tikzpicture}
	\coordinate (or) at (0,0) node[below]{$0$};
	\coordinate (y2) at (0,3.5);
      	\coordinate (x2) at (5,0);
      	\draw[<-,dashed] (y2) node[above] {$\tilde{w}$} -- (or) ;
	\draw[thick] (-5,2.5) coordinate (ha1)  -- (6,2.5) coordinate (ha2) node[right] {$h_{\tilde{w},\tilde{b}}$};
	\draw[thick] (-5,3.5) coordinate (ho1) -- (6,0) coordinate (ho2) node[right] {$h_{w,b}$};
	\draw (-5,2.9) node [right] {$Ia$};
	\draw (4,2) node [right] {$IIa$};
	\draw (4,1) node [right] {$IIb$};
	\draw (-2,3.5) node [right] {$III$};
	\draw (-2,1.5) node [right] {$IV$};
	\coordinate (int) at (intersection of ha1--ha2 and ho1--ho2);
	\coordinate (pro) at ($(ho1)!(0,0)!0:(ho2)$);
	\coordinate (proend) at (intersection of ha1--ha2 and or-- pro);
	\coordinate (propar) at ($(proend)!(x2)!90:(or)$);
	\coordinate (intho) at (intersection of ho1--ho2 and or--y2);
	\coordinate (intha) at (intersection of ha1--ha2 and or--y2);
	\draw[->,dashed] (or)--($2*(pro)$) node[right,above] {$w$};
	\draw[thick,blue] (intho)--(intha) node[midway,right]{${\rho}_{\tilde{w}}$};
	\draw [thick,blue] (pro) -- (proend) node [midway,right] {${\rho}_{w}$};
	\pic [pic text=$\cdot$, pic text options={text=black},text opacity=10,draw,angle radius=2mm, fill opacity=0.1]{angle=or--pro--ho2};
	\pic [pic text=$\cdot$, pic text options={text=black},text opacity=10,draw,angle radius=2mm, fill opacity=0.1]{angle=or--proend--propar};
	\pic [pic text=${\beta}$, pic text options={text=blue},text opacity=10,draw,angle radius=25mm, fill opacity=0.1,blue]{angle=propar--proend--ha2};
	\pic [pic text=${\beta}$, pic text options={text=blue},text opacity=10,draw,angle radius=25mm, fill opacity=0.1,blue]{angle=ho1--int--ha1};
	\coordinate (propar) at ($(proend)!(x2)!90:(or)$);
	\draw[dashed] (proend) -- ($(propar)-0.1*(proend)+0.1*(propar)$) ;
	\fill (proend) circle (1.75pt) ;
	\draw (proend) node [above right] {$0_{IIa}$};
	\fill (int) circle (1.75pt) ;
	\draw (int) node [above right] {$0_{Ia}$};
 \end{tikzpicture}
 \caption{ \textit{Approximation of the hyperplane $h_{w,b}$ by $h_{\tilde{w},\tilde{b}}$, sketched in the span of their normals. 
The angle distortion ${\beta}$ between the normal vectors and the differences $\rho_w$ and $\rho_{\tilde{w}}$ are bounded in terms of $\gamma$ and $\rho$. 
The hyperplanes separate the Regions $III$ and $IV$ with larger intersection angle from the Regions $I$ and $II$.
We further split the Regions $I$ and $II$ into two parts $a$ and $b$ by dashed hyperplanes orthogonal to the respective axis.
The Region $Ia$ and $IIa$ are cones, which roots are denoted by $0_{Ia}$ and $0_{IIa}$.
In the sketched case the Region $Ib$ vanishes, since the bordering dashed line equals the projected hyperplane $h_{\tilde{w},\tilde{b}}$.}}
 \label{fig:Hyperplanes}
\end{figure}
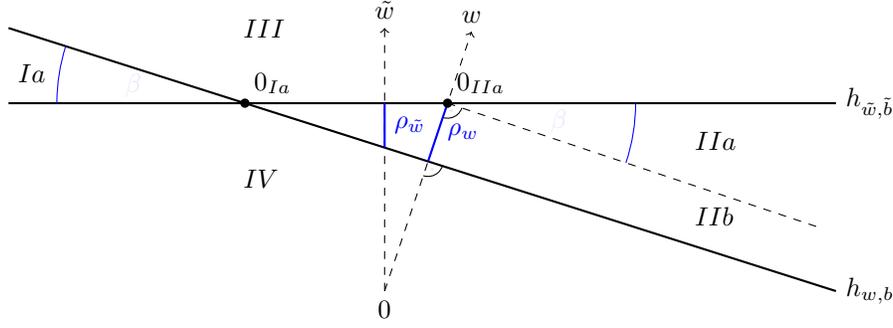

\begin{proof}
Let us take a parameter vector $(w,b,\kappa)\in P$ and set for now $\lambda=1$.
By definition, there exists an unit vector $\tilde{w}$ in the angular covering $N_{\gamma}(\mathbb{S}^{d-1},\angle)$, such that
\begin{align*}
\beta := \angle(w,\tilde{w})\leq \gamma \, .
\end{align*}
The projection of the hyperplanes $h_{w,b}$ and $h_{\tilde{w},\tilde{b}}$ orthogonally onto $U:=\mathrm{span}(w,\tilde{w})$ thus yields, which intersect with the angle $\beta$ (see Figure \ref{fig:Hyperplanes}).
For any $\tilde{b}\in \rr$, the hyperplanes $h_{w,b}$ and $h_{\tilde{w},\tilde{b}}$ intersect with the axes $\{ a \cdot \frac{w}{\|w\|}\, , \, a\in \rr \}$ and $\{ a \cdot \tilde{w}\, , \, a \in \rr \}$  at:
\begin{center}
\begin{tabular}{l|ll} 
Intersection between & \textbf{Hyperplane} \\
\textbf{Axis} &   $h_{w,b}$  & $h_{\tilde{w},\tilde{b}}$  \\
\hline
$\{ a \cdot \frac{w}{\|w\|}\, , \, a \in \rr \}$ 	& $a= - \frac{b}{\|w\|}$ 			& $a= - \frac{\tilde{b}}{\cos(\beta)}$  \\
$\{ a \cdot \tilde{w}\, , \, a\in \rr \}$    		& $a= - \frac{b}{\|w\| \cos(\beta)}$	& $a= - \tilde{b}$
\end{tabular}
\end{center}

By construction, there exists $\tilde{b} \in \{-c_{\rho} \rho,(-c_{\rho}+1)\rho,...,c_{\rho} \rho\}$ such that
\begin{align*}
\Big| \frac{b}{\|w\|} - \tilde{b} \Big| \leq \rho \, .
\end{align*}
It follows for the distances $\rho_w$ and $\rho_{\tilde{w}}$ of the axes with the hyperplanes $h_{w,b}$ and $h_{\tilde{w},\tilde{b}}$, that
\begin{align*}
{\rho}_{w} & := \Big| \frac{b}{\|w\|\cos(\beta)} - \tilde{b} \Big| \leq \rho + \left( \frac{1}{\cos(\beta)}-1\right)\frac{|b|}{\|w\|} \leq \rho \left[1 + \left( \frac{1}{\cos(\beta)}-1\right) c_{\rho}\right]  \\
{\rho}_{\tilde{w}} & := \Big| \frac{b}{\|w\|} - \frac{\tilde{b}}{\cos(\beta)} \Big| \leq  \rho + \left( \frac{1}{\cos(\beta)}-1\right)\tilde{b} \leq \rho \left[1+ \left( \frac{1}{\cos(\beta)}-1\right) c_{\rho}\right] \,.
\end{align*}

With this we provided a discretization scheme for the hyperplane $\{ h_{w,b} \, : \, (w,b,\kappa)\in P\}$.
We continue with the approximation of the gradients of the neurons, which are given by $\kappa w$ (Figure~\ref{fig:Gradientapproximation}) in the active region.
By determining the hyperplane $h_{\tilde{w},\tilde{b}}$ of the approximating neuron, we have already chosen the normalized weight $\tilde{w} \in \mathbb{S}^{d-1}$.
We are left with the choice of a rescaling parameter $\lambda \in \rr$,
\begin{align*}
(\tilde{w},\tilde{b},\kappa)\rightarrow (\lambda \tilde{w},\lambda \tilde{b},\kappa) \, ,
\end{align*}
which leaves the hyperplane unchanged. 

Notice that we assumed $\|w\|\leq c_{w} $ for $(w,b,\kappa)\in P$ and that we have set $c_{\delta} = \Big\lfloor \frac{c_{w}}{\delta} \Big\rfloor$.
We can thus choose $\lambda \in \{ 0,\delta,2\delta, \dots , c_{\delta} \delta \}$ such that
\begin{align*}
\big| \lambda - \|w\| \big| \leq \delta \, .
\end{align*}
This implies that
\begin{align*}
\|w -\lambda \tilde{w}\|^2 = \|w\|^2-2 \lambda \braket{w,\tilde{w}} \lambda  + \lambda^2 = (\lambda - \|w\|)^2 + 2(1-\cos(\beta))\lambda \|w \| \leq {\delta^2} + 2 (1-\cos(\beta))c_w^2 \, .
\end{align*}
We can furthermore bound the difference of the neurons at the origin $0$ as follows:
\begin{align*}
\big| \phi_{(w,b,\kappa)}(0)-\phi_{(\lambda \tilde{w},\lambda \tilde{b},\kappa)}(0)\big| \leq \big| b- \lambda \tilde{b} \big| \leq \Big| \frac{b}{\|w\|}- \tilde{b} \Big| \cdot \|w\|   + \Big| \|w\| - \lambda \Big| \cdot \tilde{b} \leq {\|w\| \rho+ \tilde{b} \delta} \leq \delta\rho ({c_{\delta}+c_{\rho}})
\end{align*}
For each neuron parameterized by $p=(w,b,\kappa)\in P$, we have thus determined an approximating neuron parameterized by $\tilde{p}=(\lambda \tilde{w},\lambda \tilde{b},\tilde{\kappa})$, which satisfies the in Lemma~\ref{lem:Approximationscheme} stated properties.
\end{proof}

\vspace{0.1cm}
\section{Sub-Gaussian Covering Numbers for $\mathrm{ReLU}$ Neurons}\label{app:Covering}

Following Assumption~\ref{ass:statistic}, we now assume a standard Gaussian probability distribution $\mu_x$ on the input space $\rr^d$ of the neurons.
Any function on $\rr^d$ can be regarded as a random variable with the \emph{Sub-Gaussian norm} \cite{vershynin_high-dimensional_2018}
\begin{align}\label{def:sgnorm}
\|f\|_{\psi_2} := \inf \Big\{ C_{\psi_2} \geq 0 \, : \, E_{C_{\psi_2}}(f):=\mathbb{E}_{\mu_x}\left[\exp\left( \frac{|f(x)|^2}{C_{\psi_2}^2} \right) \right]\leq 2 \Big\} \, .
\end{align}
In the following, we compute the Sub-Gaussian distance between the neurons $\phi_{(w,b,\kappa)}$ and $\phi_{(\lambda \tilde{w}, \lambda \tilde{b},\tilde{\kappa})}$, which is the Sub-Gaussian norm of the function 
\begin{align*}
g(x):=|\phi_{(w,b,\kappa)}(x)-\phi_{(\lambda\tilde{w},\lambda\tilde{b},\tilde{\kappa})}(x)|\, .
\end{align*} 

For given parameter sets $P\subset \rr^d \times \rr\times \{\pm 1\}$ and $\epsilon\geq 0$ we will then determine finite sets $\tilde{P}\subset \rr^d \times \rr\times \{\pm 1\}$, such that, for each $p\in P$, there exists $\tilde{p}\in\tilde{P}$ satisfying $\|\phi_p-\phi_{\tilde{p}}\|_{\psi_2}\leq \epsilon$.
As we have introduced in Section~\ref{sec:size}, we refer to such sets as $\epsilon$-nets of $\Phi_{\hat{P}}:=\Phi_P \cup \Phi_{\tilde{P}}$ in the Sub-Gaussian norm.
Analogously to angular covering numbers, we call the minimum of the cardinality of such $\epsilon$-nets the \emph{Sub-Gaussian covering number} $\mathcal{N}(\Phi_{\hat{P}},d_{\psi_2},\epsilon)$.

As a first step, we simplify the error integral $E_{C_{\psi_2}}(g)$ for any $C_{\psi_2}\geq 0$.
We notice, that the function $g$ only depends on the orthogonal projection $u:=P_{\mathrm{span}(w,\tilde{w})}x$ of $x$ onto the span of the normals $w$ and $\tilde{w}$ of the hyperplanes $h_{w,b}$ and $h_{\tilde{w},\tilde{b}}$.
With $v=x-u$ we then obtain
\begin{align}
E_{C_{\psi_2}}(g) &=\frac{1}{(2\pi)^{d/2}} \int_{x \in \rr^d} \exp\left[ \frac{| g(u+v)|^2}{C^2_{\psi_2}} \right]   \exp\left[- \frac{\|u\|^2 + \|v\|^2}{2}\right] du dv \nonumber \\
&=\frac{1}{2\pi} \int_{u \in \mathrm{span}\{w,\tilde{w} \}} \exp\left[ \frac{| g(u)|^2}{C^2_{\psi_2}} \right]   \exp\left[- \frac{\|u\|^2 }{2}\right] du \, .  \label{def:E} 
\end{align}

\begin{figure}[t]
\centering
\begin{tikzpicture}
	\coordinate (or) at (0,0) node[below]{$0$};
	\coordinate (y2) at (0,3.5);
      	\coordinate (x2) at (5,0);
      	\draw[<-,dashed] (y2) node[above] {$\tilde{w}$} -- (or);
	\draw[thick] (-5,2.5) coordinate (ha1)  -- (6,2.5) coordinate (ha2) node[right] {$h_{\lambda \tilde{w},\lambda \tilde{b}}$};
	\draw[thick] (-5,3.5) coordinate (ho1) -- (6,0) coordinate (ho2) node[right] {$h_{w,b}$};
	\draw (-3.7,2.7) node [right] {$I$};
	\draw (4.2,2) node [right] {$IIa$};
	\draw (4.2,1) node [right] {$IIb$};
	\draw (-2,3.5) node [right] {$III$};
	\draw (-2,1.5) node [right] {$IV$};
	\coordinate (int) at (intersection of ha1--ha2 and ho1--ho2);
	\coordinate (pro) at ($(ho1)!(0,0)!0:(ho2)$);
	\coordinate (proend) at (intersection of ha1--ha2 and or-- pro);
	\coordinate (propar) at ($(proend)!(x2)!90:(or)$);
	\coordinate (intho) at (intersection of ho1--ho2 and or--y2);
	\coordinate (intha) at (intersection of ha1--ha2 and or--y2);
	\draw[<-] (-4.3,2.65) -- (-4.3,3.28) node[midway,left] {$\lambda\kappa \tilde{w}$};
	\draw[<-] (4.15,1.97) -- (4,1.47) node[midway,left] {$\kappa w$};
	\draw[<-] (2.65,2.45) -- (2.5,1.95) node[midway,left] {$\kappa w$};
	\draw[<-] (4.15,1.15) -- (4,0.65) node[midway,left] {$\kappa w$};
	\coordinate (vecdif) at (0.3,0.1);
	\coordinate (an1) at (-4.3,3.28);
	\coordinate (an2) at (-3,2.87);
	\coordinate (an3) at (2,3.3);
	\coordinate (an4) at (3,2.87);
	\draw[->] (an1)-- ($(an1)+(vecdif)$) node[above] {$\kappa(w-\lambda\tilde{w})$};
	\draw[->] (an2)-- ($(an2)+(vecdif)$) node[above] {$\kappa(w-\lambda\tilde{w})$};
	\draw[->] (an3)-- ($(an3)+(vecdif)$) node[above] {$\kappa(w-\lambda\tilde{w})$};
	\draw[->] (an4)-- ($(an4)+(vecdif)$) node[above] {$\kappa(w-\lambda\tilde{w})$};	
	\draw[->,dashed] (or)--($2*(pro)$) node[right,above] {$w$};
	\coordinate (propar) at ($(proend)!(x2)!90:(or)$);
	\draw[dashed] (proend) -- ($(propar)-0.1*(proend)+0.1*(propar)$) ;
 \end{tikzpicture}
 \caption{Sketch of the function $g:=|\phi_{(w,b,\kappa)}-\phi_{(\tilde{w},\tilde{b},\kappa)}|$, which is linear in each Region $I-IV$. The gradient of the function is sketched by solid arrows, which are constant in each region. }
 \label{fig:Errorintegral}
\end{figure}
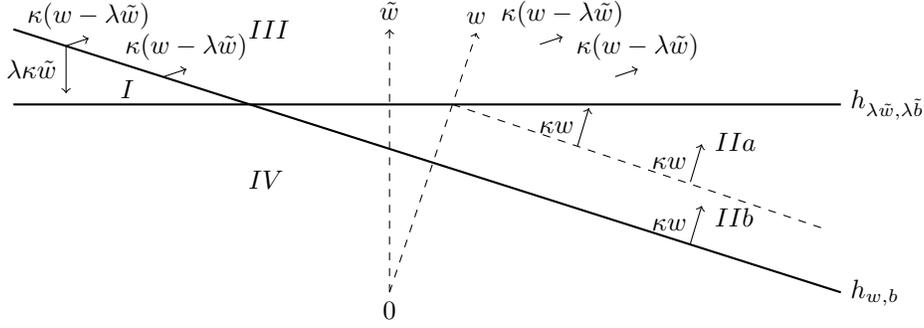

\subsection{Sub-Gaussian Radius of the Neurons}

We now discuss the \emph{Sub-Gaussian radius} of $\Phi_P$, which we defined as
\begin{align*}
\Delta(\Phi_P) := \sup_{p \in P} \| \phi_p \|_{\psi_2} \,.
\end{align*}
Under mild assumptions on the parameterizing set $P$, we derive a bound for $\Delta(\Phi_P)$ in the next Theorem.

\begin{theorem}\label{the:radius}
Let $P\subset \rr^{d}\times \rr \times \{\pm 1\}$ and $c_w\geq0$ be such that for each $(w,b,\kappa)\in P$ we have
\begin{align} \label{def:radiusassumptions}
\|w\| \leq c_w \quad \text{and} \quad \frac{b}{\|w\|} \leq  \sqrt{\ln2} \, .
\end{align}
Then we have
\begin{align*} 
\Delta(\Phi_P) \leq 2 c_w\, .
\end{align*}
\end{theorem}
\begin{proof}
In case of a parameter tuple $p=(w,b,\kappa)$ satisfying \eqref{def:radiusassumptions} we have to show $E_{2c_w}(\phi_p) \leq 2$. 
The function $\phi_p$ does for each $p=(w,b,\kappa)\in P$ not depend on directions orthogonal to $w$ and thus only depends on the orthogonal projection $u:=P_{\mathrm{span}(w)}x$.
First we estimate the term \eqref{def:E} in the case $r:=\frac{b}{\|w\|} > 0$ as
\begin{align} 
E_{2c_w}(\phi_p) & =\frac{1}{\sqrt{2\pi}} \int_{u = -r}^{\infty} \exp\left[ \frac{(u+r)^2\|w\|^2}{4c_w^2} - \frac{u^2 }{2}\right] du  +\frac{1}{\sqrt{2\pi}} \int_{u = -\infty}^{-r} \exp\left[ - \frac{u^2 }{2}\right] du \nonumber \\
& \leq \frac{1}{\sqrt{2\pi}} \int_{u = -\infty}^{\infty} \exp\left[ \frac{(u+r)^2}{4} - \frac{u^2 }{2}\right] du =\frac{1}{\sqrt{2\pi}}  \int_{u = -\infty}^{\infty} \exp \left[ - \frac{(u-r)^2}{4} + \frac{r^2}{2} \right] du \nonumber \\
& = \frac{\sqrt{2}\exp\left[ \frac{r^2}{2}\right] }{\sqrt{2\pi}}  \int_{u = -\infty}^{\infty} \exp \left[ - \frac{u^2}{2}  \right] du \leq \sqrt{2} \exp\left[ \frac{r^2}{2}\right] \label{equ:singleintegral} \, .
\end{align}
Since we have $r \leq {\sqrt{\ln2}}$ by assumption, $E_{2c_w}(\phi_p) \leq 2$ holds. 

We continue with the case $r\leq 0$, where we estimate the integral \eqref{equ:singleintegral} by
\begin{align*}
E_{2c_w}(\phi_p) & =\frac{1}{\sqrt{2\pi}} \int_{u = -r}^{\infty} \exp\left[ \frac{(u+r)^2}{4} - \frac{u^2}{2} \right] du  \leq \frac{1}{\sqrt{2\pi}} \int_{u = 0}^{\infty} \exp\left[ -\frac{u^2}{4} \right] du = \frac{1}{\sqrt{2}} \, .
\end{align*}
In both cases we have $E_{2c_w}(\phi_p)\leq 2$, which finishes the proof.
\end{proof}

\begin{remark} \label{rem:coveringone}
Let $P$ satisfy the assumptions of Theorem \ref{the:radius} for some $c_w\geq 0$.
Theorem~\ref{the:radius} then implied, that for any $\epsilon \geq 2c_w$, the set $\{\phi_{(0,0,1)}\}$ is an $\epsilon$-net of $\Phi_P \cup \{\phi_{(0,0,1)}\}$.
We thus have for any $\epsilon \geq 2c_w$ that
\begin{align*}
\mathcal{N}(\Phi_P \cup \{\phi_{(0,0,1)}\}\, , \,d_{\psi_2},\epsilon) = 1  \,.
\end{align*}
\end{remark}

\begin{remark}
A more general bound can be determined by the concentration of measure phenomenon (for an overview see \cite{ledoux_concentration_2005}) for Gaussian distributions, since by construction the $\mathrm{ReLU}$ neurons are Lipschitz continuous with the constant $\|\phi_{(w,b,\kappa)}\|_{\mathrm{Lip}}=\|w\|$. 
This implies that for any $p=(w,b,\kappa)\in \rr^d \times \rr \times \{\pm 1\}$ we obtain
\begin{align*} 
\| \phi_p \|_{\psi_2} \leq \| \phi_p - \mathbb{E}[\phi_p] \|_{\psi_2} + \|\mathbb{E}[\phi_p] \|_{\psi_2} \leq \|\phi_p\|_{\mathrm{Lip}} + \frac{\big|\mathbb{E}[\phi_p]\big|}{\ln2} \, .
\end{align*}
The expectation of $\phi_p$ is then given by
\begin{align*} 
\mathbb{E}[\phi_p] &=\frac{1}{\sqrt{2\pi}} \int_{-\frac{b}{\|w\|}}^{\infty} \|w\| (u+\frac{b}{\|w\|}) \exp\left[ -\frac{u^2}{2}\right] du  \\
&= b \frac{1}{\sqrt{2\pi}} \int_{-\frac{b}{\|w\|}}^{\infty} \exp\left[ -\frac{u^2}{2}\right] du + \frac{\|w\|}{2} \frac{1}{\sqrt{2\pi}} \exp\left[ - \frac{b^2}{2\|w\|^2}\right] \, .
\end{align*}
This allows us to conclude
\begin{align*} 
\| \phi_p \|_{\psi_2} &\leq \|w\| + \frac{|b|}{\ln2} + \frac{\|w\|}{2}  \frac{1}{\sqrt{2\pi}}\exp\left[ - \frac{b^2}{2\|w\|^2} \right] 
= \|w \| \left[ 1+ \frac{1}{2\sqrt{2\pi}}\exp\left[ - \frac{b^2}{2\|w\|^2} \right]  + \frac{|b|}{\|w\|} \frac{1}{\ln2} \right] \, .
\end{align*}
We notice that this bound holds for arbitrary parameters $p$.
However, in the case $\frac{b}{\|w\|}\leq {\ln2}$, it is less accurate compared to Theorem \ref{the:radius}.
\end{remark}

\subsection{The Case of Vanishing Biases} \label{sec:B0}

Our technique to bound the error integral $E_{C_{\psi_2}}(g)$ is the split of $g$ into linear functions, which are supported on disjoint regions.
As depicted in Figure~\ref{fig:Hyperplanes}, we partition $\mathrm{span}(w,\tilde{w})$ by the hyperplanes $h_{w,b}$ and $h_{\tilde{w},\tilde{b}}$ into the Regions $I-IV$. The Regions $I$ and $II$ span by constructions the smaller angle $\beta$, compared to the Regions $III$ and $IV$.
As sketched in Figure~\ref{fig:Errorintegral}, we further define the Region $IV$ such that $g$ vanishes on it.
We will now continue the estimation of the error integral \ref{def:E} in the special case of vanishing biases $b$ and proof the following Theorem.
After having discussed this case we will then refine the chosen region partition to treat the generic case of nonvanishing biases.

\begin{theorem}\label{the:coveringvanishing}
Let $P\subset \rr^{d}\times \{0\}\times \{\pm 1\}$ and $c_w\geq 0$ be such that, for each $(w,0)\in P$, we have
\begin{align*} 
\|w\| \leq c_w \,.
\end{align*}
Then there exists $\hat{P}$ with $P\subset \hat{P}$ such that
\begin{align*} 
\mathcal{N}(\Phi_{\hat{P}},d_{\psi_2},\epsilon) \leq 2 \, \Big\lfloor \frac{2\sqrt{2} c_w }{\epsilon} + 1\Big\rfloor \, \left( 1 + \frac{1}{\sin(\frac{\epsilon}{4\sqrt{2}c_w})} \right)^d \, .
\end{align*}
\end{theorem}

\begin{proof}
Since by assumption we have $b=0$ for each neuron indexed by $p\in P$, we can choose $\tilde{b}=0$ for an approximating neuron.
This implies, that the hyperplanes $h_{w,b}$ and $h_{\tilde{w},\tilde{b}}$ intersect at $0\in\rr^d$. 
In this case the Regions $I-IV$ (see Figure~\ref{fig:Errorintegral}) of $\mathrm{span}(w,\tilde{w})$, where $g$ behaves linear, are distinguishable only by an angle coordinate.
We therefore derive techniques to estimate the integral $E_{C_{\psi_2}}(g)$ in the polar coordinates $r:=\|u\|$ and $\alpha=\angle(u,\tilde{w})$. 
First of all, we find a function $f:[0,2\pi]\rightarrow \rr$, such that 
\begin{align} \label{eq:Polarfunction}
g(x)=r\cdot f(\alpha) \, .
\end{align}
Let $e_r(\alpha)$ be the unit vector of the coordinate $r$ at given angle $\alpha$ and $\nabla g({\alpha})$ be the gradient of $g$ at $e_r(\alpha)$, then we have
\begin{align*}
f(\alpha)=\braket{\nabla g(\alpha),e_r(\alpha)}=\| \nabla g(\alpha) \| \cos \angle(\nabla g(\alpha),e_r(\alpha))\, .
\end{align*}
We notice, that $f(\alpha)$ is with this also defined at the intersection of two regions, since the scalar product of $e_{\alpha}$ with the regions gradients is identical at their intersection.

If we now assume $C^2_{\psi_2} \geq 4f^2(\alpha)$, then we obtain
\begin{align*} 
E_{C_{\psi_2}}(g)&=\frac{1}{2\pi} \int_{\alpha=0}^{2\pi} \int_{r=0}^{\infty}  \exp \left[ \frac{r^2 f^2(\alpha)}{C^2_{\psi_2}} - \frac{r^2}{2} \right] \,r\, d \alpha dr \\
&= \frac{1}{4\pi} \int_{\alpha=0}^{2\pi} \frac{2C^2_{\psi_2}}{2f^2(\alpha)-C^2_{\psi_2}} \exp \left[  \left( \frac{ f^2(\alpha)}{C^2_{\psi_2}} - \frac{1}{2} \right) r^2 \right]_{r=0}^{r=\infty} d \alpha \\
&= \frac{1}{2\pi} \int_{\alpha=0}^{2\pi} \frac{C^2_{\psi_2}}{C^2_{\psi_2}-2f^2(\alpha)} d \alpha \leq \frac{1}{2\pi} \int_{\alpha=0}^{2\pi} \frac{4f^2(\alpha)}{4f^2(\alpha)-2f^2(\alpha)} d \alpha =2 \, .
\end{align*}

To construct an $\epsilon$-net of $P$ in the Sub-Gaussian norm we apply the approximation scheme introduced in Lemma~\ref{lem:Approximationscheme}, where we set $c_b=0$ and hence $c_{\rho}=0$.
In this scheme, any neuron $\phi_{(w,b,\kappa)}$ with $(w,b,\kappa)\in P$ is approximated by a neuron $\phi_{(\lambda \tilde{w},\lambda \tilde{b},\tilde{\kappa})}$, where we take $\tilde{\kappa}=\kappa$.
To derive conditions on $\delta,c_{\delta}$ and $\gamma$ to guarantee Sub-Gaussian distances smaller than $\epsilon$, we rewrite the condition $C^2_{\psi_2} \geq 4f^2(\alpha)$ in the different Regions as:
\begin{center}
\begin{tabular}{ll} 
(Region I:) & $  \lambda \cdot \sin \beta \leq \frac{C_{\psi_2}}{2}$ \\
(Region II:) & $  \|w\| \cdot \sin \beta \leq \frac{C_{\psi_2}}{2}$ \\
(Region III:)  & $  \| \lambda \tilde{w}-w\|^2 \leq  \frac{C_{\psi_2}^2}{4}$  \, .
\end{tabular} 
\end{center}

Let us now set $C_{\psi_2}:=\epsilon$ and choose $\tilde{w}$ from an angle covering set with distortion angle $\gamma=\frac{\epsilon}{2 \sqrt{2} c_w}$ (see Definiton~\ref{def:anglecoveringset}).
With Lemma \ref{lem:Anglecoveringnumber} we can choose the angle covering set $N_{\gamma}(\mathbb{S}^{d-1},\angle)$ such that
\begin{align*}
\big| N_{\gamma}(\mathbb{S}^{d-1},\angle) \big| \leq \left( 1 + \frac{1}{\sin(\frac{\epsilon}{4\sqrt{2}c_w})} \right)^d \, .
\end{align*}
With the estimation $\sin(\beta)\leq \beta$ the angle $\beta$ satisfies
\begin{align*}
c_w \sin(\beta)\leq \frac{\epsilon}{2 \sqrt{2}} \quad \text{and} \quad c_w^2 \big(1-\cos(\beta)\big)=2 c_w^2 \sin^2\left( \frac{\beta}{2} \right) \leq \frac{\epsilon^2}{16} \, .
\end{align*}
With this choice the condition $C^2_{\psi_2} \geq 4f^2(\alpha)$ holds in Region $I$ and $II$.

We furthermore choose
\begin{align*}
\delta=\frac{\epsilon}{2 \sqrt{2}} \, .
\end{align*}
With the approximation properties provided in Lemma \ref{lem:Approximationscheme}, the condition then holds also for Region $III$, since
\begin{align*}
{\delta^2}+2(1-\cos(\beta)) c_w^2  \leq \frac{C_{\psi_2}^2}{4} \, .
\end{align*}

Finally, we define a set of approximating neurons by the parameters 
\begin{align*}
\tilde{P} =\Big\{ (\lambda \tilde{w},0, \kappa) \, : \, \lambda \in  \{0,\delta,2\delta,\dots, c_{\delta} \delta \} \, , \, \tilde{w} \in N_{\gamma}(\mathbb{S}^{d-1},\angle) \, , \, \kappa \in \{ \pm 1\} \Big\} \, .
\end{align*} 
We then set $\hat{P}=P\cup \tilde{P}$ and notice that $\Phi_{\tilde{P}}$ is by construction an $\epsilon$-net of $\Phi_{\hat{P}}$ in the Sub-Gaussian norm.
We finish the proof by a bound of the Sub-Gaussian covering number of $\Phi_{\hat{P}}$ in terms of the cardinality of $\tilde{P}$, which implies
\begin{align*}
\mathcal{N}(\Phi_{\hat{P}},\| \cdot \|_{\psi_2},\epsilon) \leq \big| \tilde{P} \big| = 2 \Big\lfloor \frac{2\sqrt{2} c_w }{\epsilon} +1 \Big\rfloor \left( 1 + \frac{1}{\sin(\frac{\epsilon}{4\sqrt{2}c_w})} \right)^d \, .
\end{align*}
\end{proof}

\subsection{The Case of Non-Vanishing Biases}

Let us now generalize the analysis to neurons with non-vanishing biases $b$, which corresponds to possible intersections of the hyperplanes $h_{w,b}$ and $h_{\tilde{w},\tilde{b}}$ at arbitrary points in the span of $w$ and $\tilde{w}$. 

\newpage

\begin{theorem}\label{the:coveringnonvanishing}
Let $P\subset \rr^d \times \rr \times \{\pm 1\}$, $c_{w}>0$ and $ c_b \geq 1$ be such that, for all $(w,b,\kappa)\in P$, we have
\begin{align*}
\|w\| \leq c_w \quad \text{and} \quad - c_b\leq \frac{b}{\|w\|} \leq  \sqrt{\ln2} \,
\end{align*}
and
\begin{align*}
c_b \leq\left( \sqrt{\frac{\pi}{8e^2}} - \frac{1}{8} \right) \frac{\cos(\frac{1}{4})}{1-\cos(\frac{1}{4})} \approx 3.289 \, .
\end{align*}
Then there exists $\hat{P}$ with $P\subset \hat{P}$ such that, for each $ \epsilon >0$, the Sub-Gaussian covering number of the set $\Phi_{\hat{P}}$ of neurons is bounded by
\begin{align*}
\mathcal{N}(\Phi_{\hat{P}},d_{\psi_2},\epsilon) \leq  2\,  \Big\lfloor \frac{16c_bc_w}{\epsilon} +1 \Big\rfloor \,  \Big\lfloor \frac{32c_bc_w}{\epsilon}  +1 \Big\rfloor  \, \left(1+\frac{1}{\sin(\frac{\epsilon}{16c_w})}\right)^d \, .
\end{align*}
\end{theorem}

\begin{proof}
Analogously to the proof of Theorem~\ref{the:coveringvanishing} we apply the scheme of Lemma \ref{lem:Approximationscheme} to approximate each $(w,b,\kappa)\in P$ by a $(\tilde{w},\tilde{b},\tilde{\kappa}) \in \tilde{P}$.
We notice, that the treatment of non-vanishing biases $b$ requires a refinement of the partition of $\mathrm{span}(w,\tilde{w})$ into Regions $I-IV$.
Using a hyperplane, which is parallel to $h_{w,b}$ and passes through the intersection of the axis $\{ a \cdot \frac{w}{\|w\|}\, , \, a\in \rr \}$ with $h_{\lambda \tilde{w},\lambda \tilde{b}}$, we split the Region $II$ into $IIa$ and $IIb$ (see Figure~\ref{fig:Hyperplanes}). 
By interchanging $w$ and $\tilde{w}$, we split Region $I$ similarly into $Ia$ and $Ib$. 
If the hyperplanes do not intersect, we have $\beta=0$, and the regions $I$ and $II$ are empty.
This constitutes a special case, where the error integrand of \eqref{def:E} is entirely supported on the Region $Ib=IIb$ and $III$.
In this case, the estimation of \eqref{def:E} can be performed analogously to the following.

Since we choose $\tilde{b}$ with the same sign as $b$, the axes $\{ a \cdot \tilde{w} \, , \, a\in \rr \}$ and $\{ a \cdot \frac{w}{\|w\|}\, , \, a\in \rr \}$ intersect with the hyperplanes $h_{w,b}$ and $h_{\tilde{w},\tilde{b}}$, such that the respective values $a$ of intersection with $h_{w,b}$ and $h_{\lambda \tilde{w},\lambda \tilde{b}}$ have the same sign.
This prohibits the origin $0$ to lie in the interior of the Regions $I,II$.
In the following, we will therefore only distinguish the cases $0 \in IV$ and $0 \in III$.

\textbf{Case $0\in IV$}\\
By construction of the Regions $Ia$, $IIa$ and $III$, we can estimate their error integral similar to the case $b=0$, as we show in the following.
As depicted in Figure~\ref{fig:Hyperplanes}) we denote the root of the cones $Ia$ (respectively $IIa$) by $0_{Ia}$ (respectively $0_{IIa}$).
We notice, that the Gaussian density increases pointwise when shifting the Region $Ia$ (respectively $IIa$) such that the root $0_{Ia}$ (respectively $0_{IIa}$) of the cone lies at the origin $0$.
In contrary to the case $b=0$, the function $g$ does in general not vanish at the cone roots $0_{Ia}$ (respectively $0_{IIa}$).
At the cone roots we have the function values
\begin{align*}
g(0_{Ia})& =| \phi_{(w,b,\kappa)} (0_{Ia}) - \phi_{(\lambda \tilde{w},\lambda \tilde{b},\kappa)}(0_{Ia}) | = \rho_{\tilde{w}} \cdot \lambda \cdot \cos(\beta) \leq c_w\left[ {\rho}+ \left( \frac{1}{\cos(\beta)}-1 \right) c_b \right]\quad  \text{and} \\
g(0_{IIa}) &=| \phi_{(w,b,\kappa)}(0_{IIa}) - \phi_{(\lambda \tilde{w},\lambda \tilde{b},\kappa)}(0_{IIa}) | = \rho_{w} \cdot \| w\| \cdot \cos(\beta) \leq  c_w\left[ {\rho}+ \left( \frac{1}{\cos(\beta)}-1 \right) c_b \right] \, .
\end{align*}

We keep track of the values $g(0_{Ia})$ and $g(0_{IIa})$ by splitting the function ${g=|\phi_{(w,b,\kappa)} - \phi_{(\lambda \tilde{w},\lambda \tilde{b},\kappa)}|}$ with the use of characteristic functions into linear and auxiliary parts as
\begin{align*}
g=: g_a + g_l \quad \text{where} \quad g_a:=g(0_{Ia})\cdot \chi_{Ia}+g(0_{IIa})\cdot \chi_{IIa} + g\cdot \chi_{I_b\cup IIb} \quad \text{and} \quad  g_l:= g-g_a \, .
\quad \text{and} \quad  g_l:= g-g_a \, .
\end{align*}
The Sub-Gaussian norm of $g_l$ can be bounded analogously to the case $b=0$ discussed in Section~\ref{sec:B0}.
To bound the Sub-Gaussian norm of $g_{a}$ we notice
\begin{align*}
\| g_{a} \|_{\infty}  = \max(g(0_{Ia}),g(0_{IIa})) \, . 
\end{align*}
We now apply the monotony of the exponential function to obtain
\begin{align}
\|g_a\|_{\psi_2}\leq\frac{\|g_a\|_{\infty}}{\sqrt{\ln 2}} = \frac{c_w}{\sqrt{\ln 2}} \left[ \rho + \left( \frac{1}{\cos(\beta)}-1 \right)  c_b \right]
\label{eq:ga}
\end{align}
With the inequality \eqref{eq:ga} we have provided tools to bound the Sub-Gaussian norm of the auxiliary term $g_a$ in case of $0\in IV$.
Before we apply these tools to determine choices of the discretization parameter, which leads to $\epsilon$-nets in the Sub-Gaussian norm, we will in the following first discuss the case $0\in III$ and provide analogous bounds.

\textbf{Case $0\in III$}\\
If the origin $0$ lies in the Region $III$, we cannot guarantee $g(0)=0$.
In this case, we split the function $g$ into linear and auxiliary terms on the different regions by
\begin{align} \label{eq:linpartsIII}
g=: g_{a1} + g_{a2} + g_l, \, \text{where} \quad & g_l= \big(g-g(0)\big) \chi_{III} + g \chi_{Ia\cup IIa},  \\
& g_{a1}=  g(0) \chi_{III} + g\big((\tilde{b}+\frac{\rho_{\tilde{w}}}{2}) \tilde{w}\big) \chi_{I_b} +g\big((\frac{b}{\|w\|}+\frac{\rho_w}{2}) \frac{w}{\|w\|} \big) \chi_{II_b},  \label{def:ga1} \\
&  g_{a2}= \Big(g-g\big((\tilde{b}+\frac{\rho_{\tilde{w}}}{2}) \tilde{w}\big) \Big) \chi_{Ib} + \Big(g-g\big((\frac{b}{\|w\|}+\frac{\rho_w}{2}) \frac{w}{\|w\|} \big) \Big) \chi_{IIb}.  \label{def:ga2}
\end{align}
To provide estimates of the Sub-Gaussian norm of $g$, we again notice that $g_l$ can be discussed analogously to the case $b=0$ in Section~\ref{sec:B0}.

We continue with the estimation of the Sub-Gaussian norm of $g_{a1}$.
The probability weights of the regions supporting $g_{a1}$ is bounded as
\begin{align*}
\mu_x(I_b)\leq \frac{\rho_{\tilde{w}}}{\sqrt{2\pi}} \, , \quad \mu_x(II_b)\leq  \frac{\rho_{w}}{\sqrt{2\pi}} \quad \text{and} \quad  \mu_x(III) \leq  1\, .
\end{align*}
We further notice that the terms in the decomposition \eqref{def:ga1} of $g_{a1}$ have disjoint supports and obtain
\begin{align}
E_{C_{\psi_2}}(g_{a1}) & = \mathbb{E} \exp \left[ \frac{g_a^2}{C_{\psi_2}^2} \right] \nonumber \\
& \leq \frac{\rho_{\tilde{w}}}{\sqrt{2\pi}}\exp \left[ \frac{g\big((\tilde{b}+\frac{\rho_{\tilde{w}}}{2}) \tilde{w}\big)^2  }{C_{\psi_2}^2} \right] + \frac{\rho_{w}}{\sqrt{2\pi}}\exp \left[ \frac{g\big((\frac{b}{\|w\|}+\frac{\rho_w}{2}) \frac{w}{\|w\|} \big)^2  }{C_{\psi_2}^2} \right] + \exp \left[ \frac{g(0)^2  }{C_{\psi_2}^2} \right]  \label{eq:ga1} \, .
\end{align}

To next estimate the integral $E_{C_{\psi_2}}(g_{a2})$, we again notice that both terms in the decomposition \eqref{def:ga2} of $g_{a2}$ have disjoint support.
In case of $\beta \neq 0$, we integrate on the Regions $Ib$ first along the coordinate $y_1$ orthogonal to the hyperplane $h_{\tilde{w},\tilde{b}}$, where the function remains constant.
We then perform the integral along a coordinate $y_2$, which is parallel to the hyperplane $h_{\tilde{w},\tilde{b}}$.
We extend the coordinate $y_2$ to the full real space and get the estimate
\begin{align} \nonumber
\frac{1}{2\pi} \int_{x\in IIb} & \exp \left[ \frac{\Big( g(x)-g\big((\tilde{b}+\frac{\rho_{\tilde{w}}}{2}) \tilde{w} \big) \Big)^2}{C_{\psi_2}^2} - \frac{\|x\|^2}{2} \right] dx \\
& \leq \frac{1}{2\pi} \int_{y_2=-\infty}^{\infty} \int_{y_1=0}^{\frac{\rho_w}{\sin \beta}} \exp \left[  \frac{(y_2 \lambda \sin(\beta) \big)^2}{C_{\psi_2}^2} - \frac{y_2^2}{2} \right] \sin(\beta) dy_1 dy_2 \, . \label{eq:inta1} 
\end{align}
Assuming $C_{\psi_2}\geq 2\lambda \sin(\beta)$ we now introduce the integration variable $y_3=y_2 \sqrt{\frac{C_{\psi_2}^2-2\lambda^2 \sin(\beta)^2}{2C_{\psi_2}^2}}$.
This yields
\begin{align}
\frac{1}{2\pi} \int_{x\in IIb}& \exp \left[ \frac{\Big( g(x)-g\big((\tilde{b}+\frac{\rho_{\tilde{w}}}{2}) \tilde{w} \big) \Big)^2}{C_{\psi_2}^2} - \frac{\|x\|^2}{2} \right] dx \nonumber \\
& \leq \frac{\rho_w}{\sqrt{2\pi} } \sqrt{\frac{2C_{\psi_2}^2}{C_{\psi_2}^2-2\lambda^2 \sin(\beta)^2}} \frac{1}{\sqrt{2\pi}} \int_{y_3=-\infty}^{\infty} \exp \left[ - \frac{y_3^2}{2} \right] dy_3 \nonumber \\
& = \frac{\rho_w}{\sqrt{2\pi} } \sqrt{\frac{2C_{\psi_2}^2}{C_{\psi_2}^2-2\lambda^2 \sin(\beta)^2}} \leq \frac{\rho_w}{\sqrt{2\pi} } \sqrt{\frac{8\lambda^2 \sin(\beta)^2}{4\lambda^2 \sin(\beta)^2-2\lambda^2 \sin(\beta)^2}} \nonumber \\
& = \rho_w \sqrt{\frac{2}{\pi}}.
\label{eq:ga2}
\end{align}
Similarly, by interchanging $w$ and $\tilde{w}$, we estimate the contribution of the second term in the decomposition \eqref{def:ga2} to the error integral \eqref{def:E} as
\begin{align*}
\frac{1}{2\pi} \int_{x\in IIb} \exp \left[ \frac{\Big(g(x)-g\big((\frac{b}{\|w\|}+\frac{\rho_w}{2}) \frac{w}{\|w\|} \big) \Big)^2}{C_{\psi_2}^2} - \frac{\|x\|^2}{2} \right] dx \leq  \rho_{\tilde{w}} \sqrt{\frac{2}{\pi}} \, .
\end{align*}
Adding both terms and since $g_{a2}$ vanishes on the complement of $Ib\cup IIb$ in $\mathrm{span}(w,\tilde{w})$, we finally obtain
\begin{align*}
E_{C_{\psi_2}}(g_{a2}) \leq (\rho_w + \rho_{\tilde{w}}) \sqrt{\frac{2}{\pi}} + 1 \, .
\end{align*}
We remark, that the case $\beta=0$ is included in the limit $\beta \rightarrow 0$, where we have a pointwise convergence of the integrant \eqref{eq:inta1}, which is dominated by the supremum of the integrant in case $\beta=\sin^{-1}(\frac{C_{\psi_2}}{2 \lambda})$.
By Lebesques Theorem of dominated convergence, the bound $\rho_w \sqrt{\frac{2}{\pi}}$ on \eqref{eq:inta1}  thus also holds if $\beta=0$. 

\textbf{Discretization parameter choices}\\
In order to find for each $\epsilon \geq 0$ an $\epsilon$-net covering $\Phi_P$, we now choose the discretization parameters $\gamma, \delta, \rho$ in the approximation scheme of Lemma~\ref{lem:Approximationscheme} by
\begin{align}\label{discretizationchoices1}
\gamma =\frac{\epsilon}{8 c_w} \quad , \quad \delta = \frac{\epsilon}{16c_b}  \quad \text{and} \quad   \rho = \frac{\epsilon}{16c_w} \, .
\end{align}
Moreover, we set the corresponding discretization numbers to be
\begin{align*}
c_{\delta} = \Big\lfloor \frac{c_w}{\delta} \Big\rfloor = \Big\lfloor \frac{16c_bc_w}{\epsilon} \Big\rfloor   \quad \text{and} \quad c_{\rho} = \Big\lfloor \frac{c_b}{\rho} \Big\rfloor = \Big\lfloor \frac{16c_bc_w}{\epsilon} \Big\rfloor \,  .
\end{align*}
If $\epsilon > 2c_w$, we have with Remark~\ref{rem:coveringone} an $\epsilon$-net by $\{\phi_{(0,0,1)} \}$ and the stated bound on the covering number follows trivially. 
In the reminder we thus assume $\epsilon \leq 2 c_w$. 

We now construct an approximating parameter set $\tilde{P}$ as follows.
We choose $\kappa$ from $\{\pm 1\}$ and $\tilde{w}, \tilde{b}, \lambda $ as in \eqref{parameterchoices}.
With this, we found an approximating set $\Phi_{\tilde{P}}$ with cardinality
\begin{align*}
\big| \Phi_{\tilde{P}} \big| = 2 \, \Big\lfloor \frac{16c_bc_w}{\epsilon} +1 \Big\rfloor \,  \Big\lfloor \frac{32c_bc_w}{\epsilon}  +1 \Big\rfloor  \, \left(1+\frac{1}{\sin(\frac{\epsilon}{16c_w})}\right)^d \, .
\end{align*}
We then define $\hat{P}:=P\cup \tilde{P}$.
To finish the proof, it remains to prove, that $\Phi_{\tilde{P}}$ is an $\epsilon$-net for $\Phi_{\hat{P}}$ in the Sub-Gaussian norm.
It suffices to find for each $(w,b,\kappa)\in P$ a $(\tilde{w},\tilde{b},\tilde{\kappa})\in \tilde{P}$ such that
\begin{align*}
(1) \, \|g_l\|_{\psi_2} \leq \frac{\epsilon}{2} \, , \quad (2)\, \, \|g_a\|_{\psi_2} \leq \frac{\epsilon}{2} \quad \text{and} \quad (3) \,\, \|g_{a1}\|_{\psi_2},\|g_{a2}\|_{\psi_2} \leq \frac{\epsilon}{4} \, .
\end{align*}
In the reminder of the proof, we choose for any $(w,b,\kappa)\in P$ an approximating neuron by the parameters $\tilde{\kappa}=\kappa$ and $\tilde{w},\tilde{b},\lambda$ as in the approximation scheme of Lemma~\ref{lem:Approximationscheme}.
We are then left to verify the conditions $(1)-(3)$ with the discretization parameters \eqref{discretizationchoices1}.

\textbf{(1) Bounds on the linear terms}\\
To achieve a bound $\|g_l\|_{\psi_2}\leq \frac{\epsilon}{2}$, in both cases $0\in III $ and $0\in IV$, we bound $E_{\frac{\epsilon}{2}}(g_l)$ using polar coordinates $(r,\alpha)$, which were introduced in Section~\ref{sec:B0}.
We shift the Regions $Ia$ (respectively $IIa$) by the vectors $0_{Ia}$ (respectively $0_{IIa}$), such that both cone roots lie at the origin.
By construction of the region, the Gaussian density increases under this shift at each point.
We rewrite $g_l$ on the Region $III$ using the polar coordinates $(r\in \rr,\alpha \in [0,2\pi])$ and extend the angle $\alpha$ to the interval $[2\pi,2\pi+2\beta]$ parametrizing the shifted counterpart of the Regions $Ia$ and $IIa$.
By construction, the shifted function is then of the form $r \cdot f(\alpha)$ for a $f:[0,2\pi+2\beta]\rightarrow \rr$, since $g_l(\{0,0_{Ia},0_{IIa}\})=0$.
If $\frac{\epsilon^2}{4}>8f^2(\alpha)$ for all $\alpha$, we can estimate the integral analogously to case $b=0$ discussed in Section~\ref{sec:B0} by
\begin{align*}
E_{\frac{\epsilon}{2}}(g_l) & \leq \frac{1}{2\pi} \int_{\alpha=0}^{2\pi+2\beta} \frac{\epsilon^2}{\epsilon^2-8f^2(\alpha)} d \alpha \leq (1+\frac{\beta}{\pi})\, \frac{4}{3}\, .
\end{align*}
Since $\beta \leq \frac{\pi}{2}$ holds by construction, we can conclude $E_{\frac{\epsilon}{2}}(g_l) \leq 2$, provided that
\begin{align}\label{falphaboundgl}
\big|f(\alpha)\big| \leq \frac{\epsilon}{4\sqrt{2}} \, .
\end{align}
Analogously to the case $b=0$, \eqref{falphaboundgl} is satisfied with the discretization parameters \eqref{discretizationchoices1} on all regions, since we have
\begin{align*}
\lambda \sin(\beta) \leq \frac{\epsilon}{4}\, , \quad \|w\| \sin(\beta) \leq \frac{\epsilon}{4} \quad \text{and} \quad {\delta^2} + 2(1-\cos(\beta)) \lambda \|w\| \leq \frac{\epsilon^2}{32} \, .
\end{align*}

\textbf{(2) Bounds on the auxiliary terms if $0\in IV$}\\
With inequality \eqref{eq:ga} we have $\|g_a\|_{\psi_2}\leq \frac{\epsilon}{2}$ if 
\begin{align}\label{eq:rho}
\rho \leq \frac{\epsilon\,  \sqrt{\ln 2 }}{4c_w } \quad \text{and} \quad
c_b\leq \frac{\sqrt{\ln 2}\epsilon \cos(\gamma)}{4c_w\big(1-\cos(\gamma)\big)} \, .
\end{align}
The first inequality in \eqref{eq:rho} is satisfied for $\rho\leq \frac{\epsilon}{8c_w}$. From $\epsilon \leq 2c_w$, it follows that $\gamma \leq \frac{1}{4}$, and the third inequality is hence satisfied, since with the stated upper bound on $c_b$ it follows
\begin{align*}
c_b \leq \inf_{0 \leq \gamma \leq \frac{1}{4}} \frac{2\sqrt{\ln 2}\gamma \cos(\gamma)}{1-\cos(\gamma)} \, .
\end{align*}

\textbf{(3) Bounds on the auxiliary terms if $0\in III$}\\
In the case $0\in III$, we apply the estimation \eqref{eq:ga1} to achieve the bound $\|g_{a1}\|_{\psi_2}\leq \frac{\epsilon}{4}$, provided that
\begin{align*}
 \frac{\rho_{\tilde{w}}}{\sqrt{2\pi}} \exp \left[ \frac{(2\rho_{\tilde{w}}c_w \cos(\beta))^2}{\epsilon^2} \right] +  \frac{\rho_{w}}{\sqrt{2\pi}} \exp \left[ \frac{(2\rho_{w} c_w \cos(\beta))^2}{\epsilon^2} \right] +  \exp \left[ \frac{ (4\delta c_b + 4\rho c_w)^2}{\epsilon^2} \, .\right] \leq 2
\end{align*}
To satisfy this bound, it is sufficient to require
\begin{align*}
 (a) \,\, \max(\rho_{\tilde{w}},\rho_w) \leq \frac{1}{2e}\sqrt{\frac{\pi}{2}}, \quad  (b) \,\, 2c_w\max(\rho_{\tilde{w}},\rho_w) \leq {\epsilon} \quad \text{and} \quad(c) \,\,\delta c_b + \rho c_w \leq \frac{\epsilon \sqrt{\ln\left( \frac{3}{2}\right) }}{4} \, .
\end{align*}
Let us now check whether those are indeed fulfilled.
With $\max(\rho_{\tilde{w}},\rho_w) \leq \rho + (\frac{1}{\cos(\beta)}-1)c_b$ and $\rho \leq \frac{1}{8}$, condition (a) is satisfied, since
\begin{align*}
c_b \leq \left( \sqrt{\frac{\pi}{8e^2}} - \frac{1}{8} \right) \frac{\cos(\frac{1}{4})}{1-\cos(\frac{1}{4})} \approx 3.289 \, .
\end{align*}
Condition (b) follows from 
\begin{align*}
\rho \leq \frac{\epsilon}{4c_w} \quad \text{and} \quad c_b \leq \frac{\epsilon}{4c_w} \frac{\cos(\frac{1}{4})}{1-\cos(\frac{1}{4})}  \leq \frac{\cos(\frac{1}{4})}{2(1-\cos(\frac{1}{4}))}  \, .
\end{align*}
Last, condition (c) is also satisfied with the choice \eqref{discretizationchoices1}, since $\frac{1}{2}<\sqrt{\ln(\frac{3}{2})}$.

We are only left to provide a bound on $\|g_{a2}\|_{\psi_2}$.
From equation \eqref{eq:ga2}, the estimation $\|g_{a2}\|_{\psi_2}\leq \frac{\epsilon}{4}$ follows, provided that
\begin{align}\label{estimationga2}
2 \lambda \sin(\beta) \leq \frac{\epsilon}{4}  \quad \text{and} \quad  \rho_{w}+\rho_{\tilde{w}} \leq \sqrt{\frac{\pi}{2}} \, .
\end{align}
The first inequality in \eqref{estimationga2} is satisfied by assumption, since $\sin(\beta) \leq \beta$.
The second inequality follows from (a). 

This finishes the proof with the observation, that $\Phi_{\tilde{P}}$ is indeed an $\epsilon$-net of $\Phi_{\hat{P}}$, provided the choices \eqref{discretizationchoices1} of the discretization parameters. 
\end{proof}

\section{Dudleys Entropy Bound on Shallow $\mathrm{ReLU}$ Networks}\label{app:Dudley}

For any parameter set $P$ we understand the set $\Phi_P$ of neurons as a metric space with the Sub-Gaussian metric $d_{\psi_2}$, which is induced by the norm $\|\cdot\|_{\psi_2}$  (see \eqref{def:sgnorm}). 
On this metric space, we can bound the Talagrand-functional (see Definition~\ref{def:Talagrand}) as follows.

\begin{theorem}\label{the:dudleyneurons}
Let $P\subset \rr^d \times \rr \times \{\pm 1\}$ satisfy the assumptions of Theorem~\ref{the:coveringnonvanishing} for some constants $c_{w}\geq 0$ and $c_b \geq 1$. Then we have
\begin{align*}
\gamma_2(\Phi_P,d_{\psi_2})\leq \frac{8 }{(2-\sqrt{2})\sqrt{\ln2}} c_w\sqrt{8c_b + d + \frac{\ln2}{4}}\, .
\end{align*}
\end{theorem}

We will prove Theorem~\ref{the:dudleyneurons} based on a generic bound on the Talagrand-functional in terms of the metric entropy (see \cite{talagrand_generic_2005} Section~1.2), which we first formulate as the following Lemma. 

\begin{lemma}[Dudleys entropy bound \cite{talagrand_generic_2005}] \label{lem:dudleygeneral}
Let $(T,d)$ be a metric space and $\mathcal{N}(T,d,\epsilon)$ its covering number for $\epsilon>0$. Then we have
\begin{align*}
\gamma_2(T,d)\leq \frac{\sqrt{2}}{(\sqrt{2}-1)\sqrt{\ln2}}  \int_{0}^{\infty} \sqrt{\ln \mathcal{N}(T,d,\epsilon)} \, d\epsilon\, .
\end{align*}
\end{lemma}
\begin{proof}
For $k\in \mathbb{N}$, we define $\epsilon_k$ as 
\begin{align*}
\epsilon_k:=\inf \big\{ \epsilon>0 \, : \,  \mathcal{N}(T,d,\epsilon) \leq 2^{(2^k)} \big\} \, . 
\end{align*}
We then define for each $k$ a sequence $(\epsilon_k^n)_{n=1}^{\infty}$ by $\epsilon_k^n = (1+\frac{1}{n}) \epsilon_k$.
For fixed $n$, there exists a sequence of $\epsilon^n_k$-nets $(N_{\epsilon^n_k})_{k=0}^{\infty}$ in the Sub-Gaussian metric, which is admissible.
We conclude, that
\begin{align} \label{dudleysum}
\gamma_2(T,d) \leq \sup\limits_{t \in T} \sum_{k=0}^{\infty} 2^{\frac{k}{2}} d(t,N_{\epsilon^n_k}) \leq \sum_{k=0}^{\infty} 2^{\frac{k}{2}}  \epsilon^n_k \leq (1+\frac{1}{n})  \sum_{k=0}^{\infty} 2^{\frac{k}{2}}  \epsilon_k\, . 
\end{align}
Since $n$ can be chosen arbitrarly in the inequality \eqref{dudleysum}, we have
\begin{align*}
\gamma_2(T,d) \leq  \sum_{k=0}^{\infty} 2^{\frac{k}{2}}  \epsilon_k\, . 
\end{align*}
As sketched in Figure~\ref{fig:dudleysketch}, we now estimate the sum on the right hand side of \eqref{dudleysum} by the integral of the function $\sqrt{\ln \mathcal{N}(T,d,\epsilon)}$ and obtain
\begin{align*}
\sum_{k=0}^{\infty} 2^{\frac{k}{2}}  \epsilon_k \leq \frac{\sqrt{2}}{(\sqrt{2}-1)\sqrt{\ln2}} \int_{0}^{\infty} \sqrt{\ln \mathcal{N}(T,d,\epsilon)} \,\, d \epsilon \, .
\end{align*}
This yields the claim of the Lemma.
\end{proof}

\begin{figure}[t]
\centering
\begin{tikzpicture}
	\fill[fill=gray!20] (0,0.8)--(0,1.6)--(3.7,1.6);
	\fill[fill=gray!20] (0,0.8)--(3.7,1.6)--(3.7,0.8);
	\draw[thick,->] (0,0) coordinate (or)  -- (9.5,0) coordinate node[right]{$\epsilon$} ;
	\draw[thick,->] (0,0) coordinate (or)  -- (0,3.5) coordinate (ha2) node[right]{$\sqrt{\ln\mathcal{N}(T,d,\epsilon)}$};
	\draw[out=0, in=-200] (0,3) to (4, 1.5);
	\draw[dashed] (0,1.6) node[left]{$\sqrt{\ln2} \, (2^{\frac{k}{2}})$}--(3.7,1.6);
	\draw[dashed] (0,0.8) node[left]{$\sqrt{\ln2} \, (2^{\frac{k-1}{2}})$}--(5.8,0.8);
	\draw[dashed] (3.7,0) node[below]{$\epsilon_k$} -- (3.7,1.6);
	\draw[out=-20, in=-190] (4,1.5) to (9, 0);
	\draw (1.75,0.75) node [above] {$\sqrt{\ln2}  \left(\frac{\sqrt{2}-1}{\sqrt{2}}\right) \, 2^{\frac{k}{2}} \epsilon_k$};
\end{tikzpicture}
 \caption{Sketched plot of the metric entropy $\sqrt{\ln\mathcal{N}(T,d,\epsilon)}$. For each $k\in\mathbb{N}_0$, a disjoint region (grey) with the area $\sqrt{\ln2}  \left(\frac{\sqrt{2}-1}{\sqrt{2}}\right) \, 2^{\frac{k}{2}} \epsilon_k$ is covered by its integral.}
\label{fig:dudleysketch}
\end{figure}
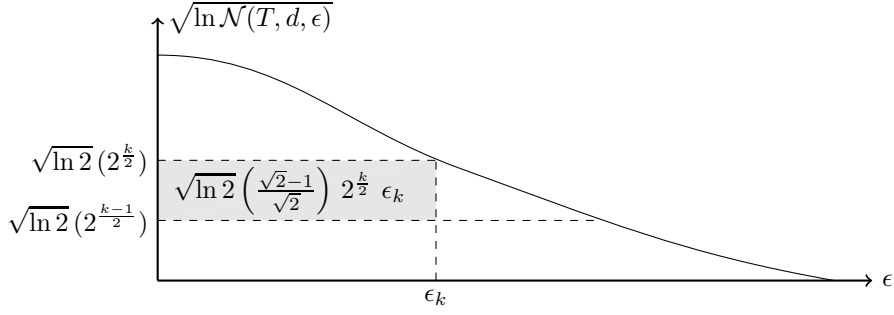

With the bounds on the Sub-Gaussian covering numbers, which were provided in Theorem~\ref{the:coveringnonvanishing}, we now prove Theorem~\ref{the:dudleyneurons}.

\begin{proof}[Proof of Theorem~\ref{the:dudleyneurons}]
As a first step, we enlarge $\Phi_P$ to $\Phi_{\hat{P}}$ by including, for all $\epsilon \geq 0$, the $\epsilon$-net $\Phi_{\tilde{P}}$, which were constructed in the proof of Theorem~\ref{the:coveringnonvanishing}.
Since it holds $P\subset \hat{P}$, we have $\gamma_2(\Phi_{P},d_{\psi_2})\leq \gamma_2(\Phi_{\hat{P}},d_{\psi_2})$.
We furthermore apply Lemma~\ref{lem:dudleygeneral} for $(T,d)=(\Phi_{\hat{P}},d_{\psi_2})$ and obtain
\begin{align}\label{neurondudleyintegral}
\gamma_2(\Phi_{P},d_{\psi_2}) & \leq  \frac{\sqrt{2}}{(\sqrt{2}-1)\sqrt{\ln2}}  \int_{0}^{\infty} \sqrt{\ln \mathcal{N}(\Phi_{\hat{P}},d_{\psi_2},\epsilon)} \, d\epsilon \, .
\end{align}
We have $\ln \big(\mathcal{N}(\Phi_{\hat{P}},d_{\psi_2},\epsilon)\big) =0$ for $\epsilon > 2c_w$, as discussed in Remark~\ref{rem:coveringone}.
This allows us to reduce the support of the integral on the right hand side of \eqref{neurondudleyintegral} to $[0,2c_w]$.
Applying Theorem~\ref{the:coveringnonvanishing} yields
\begin{align*}
\gamma_2&(\Phi_{P},d_{\psi_2})   \leq \frac{\sqrt{2}}{(\sqrt{2}-1)\sqrt{\ln2}} \int_{0}^{2c_w} \sqrt{ \ln \mathcal{N}(\Phi_{\hat{P}},d_{\psi_2},\epsilon) } \, d\epsilon \\
& \leq \frac{\sqrt{2}}{(\sqrt{2}-1)\ln2} \int_{0}^{2c_w} \sqrt{ \ln \left(\frac{32c_wc_b}{\epsilon}+1\right) +  \ln \left(\frac{16c_wc_b}{\epsilon}+1\right)     +d  \ln \left( \frac{1}{\sin\left(\frac{\epsilon}{16c_w}\right) }+1 \right) + \ln2 } \, d\epsilon \\
& \leq \frac{16\sqrt{2}  c_w}{(\sqrt{2}-1)\sqrt{\ln2}} \int_{0}^{\frac{1}{8}} \sqrt{ \ln \left(\frac{2c_b}{\epsilon}+1\right) +  \ln \left(\frac{c_b}{\epsilon}+1\right)     + d \ln \left( \frac{1}{\sin\epsilon }+1 \right) +\ln2 } \, d\epsilon \, .
\end{align*}
The Cauchy-Schwarz inequality provides then a bound on $\gamma_2(\Phi_{P},d_{\psi_2})$ as
\begin{align} \label{equ:gamma1}
 \frac{16\sqrt{2}  c_w}{(\sqrt{2}-1)\sqrt{\ln2}} \left[\int_{0}^{\frac{1}{8}} \left( \ln \left(\frac{2c_b}{\epsilon}+1\right) +  \ln \left(\frac{c_b}{\epsilon}+1\right)     + d \ln \left( \frac{1}{\sin \epsilon}+1 \right)  +\ln2\right) \, d\epsilon \right]^{\frac{1}{2}} \cdot \left[\int_{0}^{\frac{1}{8}} 1 d \epsilon\right]^{\frac{1}{2}} \, .
\end{align}
We next compute the contribution of the approximation of the bias in the first integral of \eqref{equ:gamma1} by
\begin{equation}\begin{split}\label{equ:ent1}
\int_{0}^{\frac{1}{8}} \ln \left( \frac{2c_b}{\epsilon} +1 \right) d \epsilon &= 2 c_b \int_{0}^{\frac{1}{16c_b}} \ln \left( \frac{1}{\epsilon} +1 \right) d \epsilon = 2 c_b \left[ \epsilon \ln \left( \frac{1}{\epsilon} +1\right) + \ln(\epsilon +1)\right]_{\epsilon=0}^{\frac{1}{16c_b}} \\
& = \frac{1}{8} \ln(16c_b+1) + 2c_b \ln\left( \frac{1}{16c_b} +1 \right) \, .
\end{split}\end{equation}
We proceed in a similar manner concerning the contribution of discretization of the weight norm, namely
\begin{equation}\label{equ:ent2}
\int_{0}^{\frac{1}{8}} \ln \left( \frac{c_b}{\epsilon} +1 \right) d \epsilon = \frac{1}{8} \ln(8c_b+1) + c_b \ln\left( \frac{1}{8c_b} +1 \right) \, .
\end{equation}
The third contribution results from the approximation of the normalized weight and can be bounded as 
\begin{equation}\label{equ:ent3}
d  \cdot \int_{0}^{\frac{1}{8}} \ln \left( \frac{1}{\sin\left({\epsilon}\right) } + 1\right) d \epsilon \leq \frac{d}{2}  \, .
\end{equation}
With \eqref{equ:ent1}-\eqref{equ:ent3}, the estimate $\ln(x+1)\leq x$ for $x\geq 0$ and the bound $c_b\geq1$ we conclude
\begin{equation}\begin{split} \label{eq:intestimation}
\int_{0}^{\frac{1}{8}} \left( \ln \left(\frac{2c_b}{\epsilon}+1\right) +  \ln \left(\frac{c_b}{\epsilon}+1\right)  + d \ln \left( \frac{1}{\sin \epsilon } + 1 \right)  +\ln2  \right) \, d\epsilon \leq 4c_b + \frac{d}{2} + \frac{\ln2}{8} \, .
\end{split}\end{equation}
Finally, we combine the bound \eqref{eq:intestimation} with the bound \eqref{equ:gamma1} to obtain
\begin{align*}
\gamma_2(\Phi_{P},d_{\psi_2})  & \leq \frac{8 }{(2-\sqrt{2})\sqrt{\ln2}} c_w \sqrt{8c_b + d + \frac{\ln2}{4}} \, .
\end{align*}
\end{proof}

We now continue with the analysis of shallow neural networks with a fixed number $n$ of neurons in the single hidden layer.
We first recall the notation introduced in Section~\ref{sec:assumptions}.

\begin{definition}
Let $n \in \mathbb{N}$ be a number of neurons and $\bar{P}\subset \left(\rr^d \times \rr \times \{\pm 1\}\right)^{\times n}$ a parameter set.
The set of \textbf{shallow $\mathrm{ReLU}$ networks} with parameters from $\bar{P}$ is defined as
\begin{align*}
\Phi_{\bar{P}} = \Big\{ \phi_{\bar{p}} := \sum_{i=0}^n \phi_{p_i} \, : \, \bar{p}=(p_1,\dots,p_n) \in \bar{P} \Big\} \, .
\end{align*}
\end{definition}

\begin{remark}\label{rem:shallowdef}
In the literature (see, e.g., \cite{bolcskei_optimal_2017}), shallow neural networks are often introduced as arbitrary real linear combinations of neurons.
Let us now consider a weighted sum of neurons $\phi_{p_i}$, where $p_i\in\rr^d\times \rr\times\{1\}$ and weights $\lambda_i\in\rr$ for $i\in[n]$.
Using the sign function, for each weighted neuron it holds
 \begin{align*}
 \lambda_i \phi_{(w_i,b_i,1)}(x)=\lambda_i \, \max(\braket{w_i,x}+b_i,0)&= \mathrm{sign} (\lambda_i) \max(\braket{|\lambda_i| w_i,x}+|\lambda_i| b_i,0)\\ &=  \phi_{\big(|\lambda_i| w_i,\, |\lambda_i|b_i,\mathrm{sign}(\lambda_i)\big)}(x) \, .
 \end{align*}
We apply this equality to represent the weighted sum of neurons as
\begin{align*}
\sum_{i=0}^n \lambda_i \phi_{(w_i,b_i,1)}(x) =\phi_{\big(|\lambda_1| w_1,\, |\lambda_1| b_1,\mathrm{sign}(\lambda_1)\big),\dots,\big(|\lambda_n| w_n,\,|\lambda_n| b_n,\mathrm{sign}(\lambda_n)\big) }(x) \, .
\end{align*}
Hence, each shallow $\mathrm{ReLU}$ network indeed admits a representation in the form \eqref{def:Shallow}.
\end{remark}

In the following Theorem, we derive bounds on the Sub-Gaussian covering number on shallow networks.

\begin{theorem} \label{the:coveringshallow}
Let $\bar{P}$ be a parameter set respecting Assumption~\ref{ass:parameter}. Then there exists a set $\hat{P}$ with $\bar{P}\subset \hat{P}$ and
\begin{align} \label{eq:card}
\mathcal{N}(\Phi_{\hat{P}},d_{\psi_2},\epsilon) \leq  2^n \cdot \Big\lfloor \frac{16nc_bc_w}{\epsilon} +1 \Big\rfloor^n \cdot  \Big\lfloor \frac{32nc_bc_w}{\epsilon}  +1 \Big\rfloor^n  \cdot \left(1+\frac{1}{\sin(\frac{\epsilon}{16nc_w})}\right)^{nd} \, .
\end{align}
\end{theorem}
\begin{proof}
We denote by $P$ the union of all elements $p_i$ in all tuples $(p_1,\dots,p_n) \in \bar{P}$.
We notice that the parameter bounds in Theorem~\ref{the:coveringnonvanishing} are satisfied for $P$, since $\bar{P}$ satisfies Assumption~\ref{ass:parameter} and we have
\begin{align*}
3 < \left( \sqrt{\frac{\pi}{8e^2}} - \frac{1}{8} \right) \frac{\cos(\frac{1}{4})}{1-\cos(\frac{1}{4})}   \, .
\end{align*}
Theorem~\ref{the:coveringnonvanishing} provides an $\frac{\epsilon}{n}$ net $\Phi_{\tilde{Q}}$ of a $\Phi_{\hat{Q}}$, where $\hat{Q}=P\cup \tilde{Q}$ with cardinality
\begin{align}\label{shallowcoveringestimation}
\mathcal{N}(\Phi_{\hat{Q}},d_{\psi_2},\epsilon) \leq  2\,  \Big\lfloor \frac{16nc_bc_w}{\epsilon} +1 \Big\rfloor \,  \Big\lfloor \frac{32nc_bc_w}{\epsilon}  +1 \Big\rfloor  \, \left(1+\frac{1}{\sin(\frac{\epsilon}{16nc_w})}\right)^d \, .
\end{align}
Let $\bar{p}=(p_1,...,p_n)\in \bar{P}$, then we find for each $p_i$ with $i\in[n]$ an $\tilde{p}_i\in\tilde{P}$ such that
\begin{align*}
 \| \phi_{p_i} - \phi_{\tilde{p}_i}\|_{\psi_2} \leq   \frac{\epsilon}{n} \, .
\end{align*}
This implies, that
\begin{align}\label{trianglesubgaussian}
\| \phi_{p_1,\dots,p_d} - \phi_{\tilde{p}_1,\dots,\tilde{p}_n}\|_{\psi_2}  \leq \sum_{i=0}^n \| \phi_{p_i} - \phi_{\tilde{p}_i}\|_{\psi_2} \leq   \epsilon \, .
\end{align}
We then define the tuple set $\tilde{P}=\tilde{Q}^{\times n}$ and realize, that with \eqref{trianglesubgaussian} the set $\Phi_{\tilde{P}}$ is an $\epsilon$-net of $\Phi_{\bar{P}}$.
Setting $\hat{P}=\hat{Q}^{\times n}$ finishes the proof.
\end{proof}

We continue with providing bounds on the metric entropy integral and the Talagrand-functional.

\begin{theorem} \label{the:dudleyshallow2}
Let $\bar{P}\subset \left(\rr^d \times \rr \times \{\pm 1\}\right)^{\times n}$ satisfy Assumption~\ref{ass:parameter}.
Then we have
\begin{align*}
\int_{0}^{\infty} \sqrt{\ln \mathcal{N}(\Phi_{\bar{P}},d_{\psi_2},\epsilon)} \, d\epsilon  \leq 4  n^{\frac{3}{2}} c_w \sqrt{8c_b +d+\frac{\ln2}{4} } 
\end{align*}
and
\begin{align*}
\gamma_2(\Phi_{\bar{P}},d_{\psi_2}) \leq \frac{8}{(2-\sqrt{2})\sqrt{\ln2}} n^{\frac{3}{2}} c_w  \sqrt{8c_b +d+\frac{\ln2}{4} } \, .
\end{align*}
\end{theorem}
\begin{proof}
Analogously to the proof of Theorem~\ref{the:coveringshallow}, we set $P$ to the union of all elements $p_i$ in all tuples $(p_1,\dots,p_n) \in \bar{P}$. 
We recall the bound \eqref{shallowcoveringestimation}, which reads
\begin{align*}
\mathcal{N}(\Phi_{\bar{P}},d_{\psi_2}, \epsilon) \leq  \mathcal{N}\left(\Phi_P,d_{\psi_2}, \frac{\epsilon}{n} \right)^n \, .
\end{align*}
From Theorem~\ref{the:radius} we furthermore know $\| \phi_{p_1,\dots,p_n}\|_{\psi_2} \leq 2nc_w$ for any $p_1,\dots, p_n \in P$. 
We apply both inequalities together with Dudleys entropy bound (Lemma~\ref{lem:dudleygeneral}) and estimate the Talagrand-functional as
\begin{align*}
\gamma_2(\Phi_{\bar{P}},d_{\psi_2}) &\leq \frac{\sqrt{2}}{(\sqrt{2}-1)\sqrt{\ln2}} \int_{0}^{2nc_w} \sqrt{\ln \mathcal{N}(\Phi_{\bar{P}},d_{\psi_2},\epsilon)} \, d\epsilon \\
& \leq  \frac{\sqrt{2}}{(\sqrt{2}-1)\sqrt{\ln2}} \int_{0}^{2nc_w} \sqrt{n \cdot \ln \mathcal{N} (\Phi_P,d_{\psi_2},\frac{\epsilon}{n} ) } \, d\epsilon \\
& =  \frac{\sqrt{2}}{(\sqrt{2}-1)\sqrt{\ln2}} n^{\frac{3}{2}} \int_{0}^{2c_w} \sqrt{  \ln \mathcal{N} (\Phi_P,d_{\psi_2},\epsilon ) } \, d\epsilon  \, .
\end{align*}
With the bounds \eqref{equ:gamma1} and \eqref{eq:intestimation} from the proof of Theorem~\ref{the:dudleyneurons}, we conclude that
\begin{align*}
\gamma_2(\Phi_{\bar{P}},d_{\psi_2}) \leq \frac{8}{(2-\sqrt{2})\sqrt{\ln2}} n^{\frac{3}{2}} c_w  \sqrt{8c_b +d+\frac{\ln2}{4} } \, .
\end{align*}
\end{proof}

\section{Sample Complexity for achieving $\nrip$}\label{app:NeuRIP}

As we have discussed in Section~\ref{sec:empiric}, we are interested in the concentration of the empirical risk at the expected risk. 
While the expected risk corresponds to the norm $\|\cdot\|_{\mu}$ of the space $L^2(\rr^d,\mu_x)$, we now define the seminorm, which corresponds to the empirical risk.

\begin{definition}\label{def:empiricalseminorm}
Let $x$ be a random variable with values in $\rr^d$, which follows a distribution $\mu$. Let further $x_1,...,x_m$ be independent copies of $x$. Then the \textbf{empirical norm} $\|\cdot\|_m$ is the random seminorm on $L^2(\rr^d,\mu_x)$, which is, for any $f\in L^2(\rr^d,\mu_x)$, defined as
\begin{align*}
\|f\|_m := \sqrt{\frac{1}{m} \sum_{j=1}^m f(x_j)^2}\, .
\end{align*}
We further introduce the corresponding \textbf{empirical scalar product} $\braket{\cdot, \cdot}_m$, which is, for any $f,g \in L^2(\rr^d,\mu_x)$, defined as
\begin{align*}
\braket{f,g}_m := \frac{1}{m} \sum_{j=1}^m f(x_j)g(x_j) \, .
\end{align*}
\end{definition}

\begin{remark}\label{rem:seminorm}
Subadditivity and absolute homogeneity of $\|\cdot\|_m$ follow from the Euclidean norm on the space $S\Big(L^2(\rr^d,\mu_x)\Big)=\rr^m$.
This implies, that $\|\cdot\|_m$ is indeed a seminorm for any samples $\{x_j\}_{j=1}^m$.
However, we notice that $\|f\|_m=0$ holds for any $f$ in the kernel of the sample operator $S$.
If we exclude degenerated measures $\mu_x$, the kernel of $S$ can not be trivial.
In this case, $\|\cdot\|_m$ fails to be a norm.
But, by reducing $L^2(\rr^d,\mu_x)$ to certain hypothesis functions, the kernel can be rendered trivial and the empirical norm will not vanish for any function different from zero.
\end{remark}

We now provide bounds on the sample complexity for achieving the $\nrip$ event, which we have introduced in Definition~\ref{def:neurip}.

\begin{theorem} \label{thm:NeuRIP2} 
Let $\bar{P} \subset \left(\rr^d \times \rr \times \{\pm 1\} \right)^{\times n}$ be a parameter set and $c_w\geq 0, c_b\in[1,3]$ be constants such that, for all $\bar{p}=(w_i,b_i)_{i=1}^n \in \bar{P}$ and $i\in[n]$, we have
\begin{align*}
\frac{\| w_i\|}{\| \phi_{\bar{p}} \| } \leq c_w \quad \text{and} \quad -c_b \leq \frac{b_i}{\|w_i\|} \leq \sqrt{\ln 2} \, .
\end{align*}
Then, there exist universal constants $C_1,C_2 \in \rr$ such that the following holds:
For each $u \geq 2$ and $s \in (0,1)$, $\mathrm{NeuRIP_s(\bar{P})}$ is satisfied with probability at least $1-17\exp\left[-\frac{u}{4} \right]$ provided that
\begin{align*}
& m \geq  n^3 c_w^2 \left( 8c_b+d + \frac{\ln2}{4}\right) \max \left( C_1 \frac{u}{s} \, , \, C_2 n^2c_w^2\left(\frac{u}{s}\right)^2  \right) \, .
\end{align*}
\end{theorem}

To prepare for the proof of Theorem~\ref{thm:NeuRIP2}, we notice that $\mathrm{NeuRIP}_s(\bar{P})$ is equivalent to
\begin{align} \label{eq:supRIP}
s \geq \sup_{\bar{p}\in \bar{P}} \Big|  \| \frac{\phi_{\bar{p}}}{\| \phi_{\bar{p}} \|} \|^2_m -  1  \Big|  \, .
\end{align}
We apply a chaining tool to bound the right hand side of \eqref{eq:supRIP}, which we first prove with the following Lemma.

\begin{lemma}\label{lem:supquadratical2}
Let $\Phi$ be a set of real functions and define
\begin{align*}
N(\Phi) := \frac{ \sqrt{2}}{(\sqrt{2}-1)\sqrt{\ln2}} \int_{0}^{\infty}  \sqrt{\ln\mathcal{N}\left(\Phi,d_{\psi_2}, {\epsilon} \right)} \, d \epsilon  \quad \text{and} \quad \Delta(\Phi) := \sup_{\phi \in \Phi} \| \phi\|_{\psi_2} \, .
\end{align*}
Then, for any $u\geq 2$, we have with probability at least $1 - 17 \exp\left[ - \frac{u}{4}\right]$, that
\begin{align*}
\sup_{\phi \in \Phi} \Big| \| \phi \|^2_m - \| \phi\|^2 \Big| \leq  \frac{u}{\sqrt{m}} \left[ 25 \frac{N(\Phi)}{m^{\frac{1}{4}}}   +   \sqrt{85 \Delta(\Phi) \, N(\Phi)}\right]^2   \, .
\end{align*}
\end{lemma}
\begin{proof}
We follow the proof of \cite[Theorem 5.5]{dirksen_tail_2015} and define the process
\begin{equation}\begin{split} \label{def:A}
A_{\phi} := \|\phi \|^2_m - \|\phi \|^2 \, .
\end{split}\end{equation}
As we did in the proof of Lemma~\ref{lem:dudleygeneral}, we specify an admissible sequence $(\Phi_k)_{k=0}^{\infty}$ by $\epsilon$-nets of the set $\Phi$ in the Sub-Gaussian metric. 
We take $\Phi_0 = \{0\}$ and define for all $k \in \mathbb{N}_0$ a map $\pi_k: \Phi \rightarrow \Phi_k$ such that
\begin{align*}
\pi_k(\phi) \in \mathrm{argmin}_{\tilde{\phi} \in \Phi_k} \|\phi - \tilde{\phi} \|_{\psi_2}
\end{align*}
Next, we apply \cite[Lemma 5.4]{dirksen_tail_2015}, which is itself an application of Bernstein's concentration inequality on increments of the process $(A_{\phi})_{\phi \in \Phi}$.
We further apply \cite[Lemma A.4]{dirksen_tail_2015} and have for each $u\geq 2$ with probability at least $1-17 \exp \left[- \frac{u}{4} \right]$ the following event:
For all $\phi \in \Phi$, $k \in \mathbb{N}$ with $2^{\frac{k}{2}} \leq \sqrt{m}$, we have
\begin{align*}
\Big| A_{\pi_{k-1}(\phi)}-A_{\pi_{k}(\phi)}  \Big| \leq u \frac{10 \Delta(\Phi)}{\sqrt{m}} 2^{\frac{k}{2}} \|\pi_{k-1}(\phi) - \pi_k (\phi) \|_{\psi_2} \, ,
\end{align*}
and, for all $k \in \mathbb{N}$ with $2^{\frac{k}{2}} > \sqrt{m}$, it simultaneously holds
\begin{align*}
 \|\pi_{k-1}(\phi) - \pi_k (\phi) \|_m \leq \sqrt{u}  \frac{5 }{\sqrt{m}} 2^{\frac{k}{2}}  \|\pi_{k-1}(\phi) - \pi_k (\phi) \|_{\psi_2} \, .
\end{align*}
We now split the telescope sum $A_{\phi} = \sum_{k=1}^{\infty} \left( A_{\pi_{k}(\phi)} - A_{\pi_{k-1}(\phi)} \right)$ with respect to both regimes and further follow the proof of \cite[Theorem 5.5]{dirksen_tail_2015}.
As a result, with probability at least $1-17 \exp \left[- \frac{u}{4} \right]$, we obtain
\begin{align*}
\sup_{\phi \in \Phi} | A_{\phi} | \leq  \frac{u}{\sqrt{m}} \left[ 25 \frac{N(\Phi)}{m^{\frac{1}{4}}}   +   \sqrt{85 \Delta(\Phi) \, N(\Phi)}\right]^2 \, .
\end{align*}
The statement then follows with the process \eqref{def:A}.
\end{proof}

We now apply Lemma~\ref{lem:supquadratical2} and the bounds on the metric entropy, which was provided in Theorem~\ref{the:dudleyshallow2}, to prove Theorem~\ref{thm:NeuRIP2}.

\begin{proof}[Proof of Theorem~\ref{thm:NeuRIP2}]
We define the parameter set $\hat{P}:=\{\hat{p}:=\frac{\bar{p}}{\| \phi_{\bar{p}}\|} \, : \, \bar{p} \in \bar{P} \}$ and notice that $\|\phi_{\hat{p}}\|_{\mu}=1$ for all $\hat{p}\in\hat{P}$. 
Due to the equivalence of  $\mathrm{NeuRIP_s(\bar{P})}$ to \eqref{eq:supRIP}, we only have to show
\begin{align*} 
\sup_{\hat{p}\in \hat{P}} \Big|  \| \phi_{\hat{p}} \|^2_m -  1  \Big| \leq s \, .
\end{align*}
We now apply Lemma \ref{lem:supquadratical2} and notice, that the claim of Theorem~\ref{thm:NeuRIP2} follows, provided that
\begin{align} \label{samplesizecondition}
&  \frac{u}{\sqrt{m}} \left[ 25 \frac{N(\Phi_{ \hat{P}})}{m^{\frac{1}{4}}}   +   \sqrt{85 \Delta(\Phi_{ \hat{P}}) \, N(\Phi_{ \hat{P}})}\right]^2 \leq s  \, .
\end{align}
To show, that inequality \eqref{samplesizecondition} holds, we prove
\begin{align*} 
& \sqrt{m} \geq {50 N(\Phi_{ \hat{P}})} \sqrt{\frac{u}{s}}  \quad \text{and} \quad \sqrt{m} \geq 340  \Delta(\Phi_{ \hat{P}}) \, N(\Phi_{ \hat{P}}) \frac{u}{s}  \, ,
\end{align*}
which are satisfied for
\begin{align} \label{samplecomplexitynips}
&\sqrt{m} \geq 10 N(\Phi_{ \hat{P}})\,  \max \left( 5 \sqrt{\frac{u}{s}} \, , \, 34 \Delta(\Phi_{ \hat{P}})  \frac{u}{s} \right) \, .
\end{align}

Theorem~\ref{the:radius} provides the radius of the set $\Phi_{\hat{P}}$ is estimated by
\begin{align} \label{radiuscondition}
\Delta(\Phi_{\hat{P}})=\sup_{\hat{p} \in \hat{P}}\|\phi_{\hat{p}}\|_{\psi_2} \leq 2nc_w \, .
\end{align}
Furthermore, we have with Theorem~\ref{the:dudleyshallow2} the bound
\begin{align} \label{entropycondition}
N(\Phi_{ \hat{P}}) \leq  4  n^{\frac{3}{2}} c_w \sqrt{8c_b +d+\frac{\ln2}{4} } \, .
\end{align}

We conclude, that the lower bound \eqref{samplecomplexitynips} is with \eqref{radiuscondition}, \eqref{entropycondition} and the assumption on $m$ satisfied for the universal constants
\begin{align*}
& C_1 = \frac{40^2 \cdot 2 \cdot 5^2}{(\sqrt{2}-1)^2 \ln 2} \quad \text{and} \quad C_2= \frac{40^2 \cdot 2 \cdot 64^2}{(\sqrt{2}-1)^2 \ln 2} \, .
\end{align*}
Thus, the proof is finished.
\end{proof}

\section{Uniform Bounds on the Expected Risk}\label{app:Agnostic}

We are now well prepared for the proof of generalization error bounds, which hold uniformly in the sublevel sets of the empirical risk, as we have stated in Section~\ref{sec:generalization}.
To this end, we apply concepts from \cite{mendelson_upper_2016}, which we introduce in the next Definition.

\begin{definition}\label{def:mendelson}
Let $(\Omega,\mu)$ be a probability space. For any random variable $X$ on $\Omega$, we define 
\begin{align*}
\| X \|_{(p)} = \sup_{1\leq q \leq p} \frac{\| X \|_q}{\sqrt{q}}
\end{align*}
We notice, that $\|\cdot\|_{(p)}$ is a norm and denote by $d_{(p)}$ the induced metric.
Let further $\Phi$ be a set of random variables on $\Omega$ and $u\geq 1$,$k_0\in \mathbb{N}_0$ be constants.
Then, the \textbf{$\Lambda_{k_0,u}$-functional} of $\Phi$ is defined as
\begin{align*}
\Lambda_{k_0,u}(\Phi)=\inf_{(\Phi_k)}\left[ \sup_{\phi\in \Phi} \sum_{k\geq k_0} 2^{\frac{k}{2}} d_{(u^2 2^k)} (\phi,\Phi_k) + 2^{\frac{k_0}{2}} \sup_{\phi_0 \in \Phi_0} \| \phi_0\|_{(u^2 2^{k_0})}\right] \, ,
\end{align*}
where the infimum is taken over all admissible sequences $(\Phi_k)_{k=0}^{\infty}$ in $\Phi$ (see Definition~\ref{def:Talagrand}).
\end{definition}

If the random variables in $\Phi$ have a finite Sub-Gaussian norm, as it is the case for the shallow $\mathrm{ReLU}$ networks under Assumption~\ref{ass:parameter} and \ref{ass:statistic}, one can utilize the following bound on the $\Lambda_{0,u}$-functional.

\begin{lemma}\label{lem:Mendelsonbound}
For any set $\Phi$ of functions and $u\geq 1$, we have
\begin{align*}
\Lambda_{0,u}(\Phi) \leq \sqrt{\frac{2}{e}} (\gamma_2(\Phi,d_{\psi_2}) + \Delta(\Phi) \, .
\end{align*}
\end{lemma}
\begin{proof}
For each random variable $X$ and any $p\geq 1$, \cite[Lemma A.2]{gotze_concentration_2019} implies that
\begin{align*}
\|X\|_{p} \leq 2\sqrt{\frac{2p}{e}} \, \|X\|_{\psi_2} \, .
\end{align*}
This yields
\begin{align*}
\|X\|_{(p)} \leq \sqrt{\frac{2}{e}} \,  \|X\|_{\psi_2} \, .
\end{align*}
The proof is finished by comparison of the Talagrand-functional, which is given in Definition~\ref{def:Talagrand}, with the {$\Lambda_{0,u}$-functional.}
\end{proof}

\begin{theorem}\label{the:mendelsonshallow}
Let $P\subset \rr^d \times \rr \times \{\pm 1\}$ satisfy Assumption~\ref{ass:parameter} for $c_{w}\geq 0$ and $ c_b \in [1,3]$. 
Then, for any $u\geq 1$, we have
\begin{align*}
\Lambda_{0,u}(\Phi_P)\leq \left(\frac{8 }{\sqrt{e}(\sqrt{2}-1)\sqrt{\ln2}}n^{\frac{3}{2}} +2 \right)c_w \sqrt{8c_b +d + \frac{\ln2}{4}}
\end{align*}
\end{theorem}
\begin{proof}
This result follows from Lemma~\ref{lem:Mendelsonbound} with the bounds on the Talagrand-functional provided in Theorem~\ref{the:dudleyshallow2}.
\end{proof}

In Section~\ref{sec:generalization} we have introduced the empirical risk optimization problem \eqref{def:Prym} on a parameter set $\bar{P}$ for given data $\{(x_j,y_j)\}_{j=1}^m \subset \rr^d \times \rr$.
For the further analysis of \eqref{def:Prym}, we define for $\bar{p},\bar{p}^* \in \bar{P}$ the \emph{excess risk} as
\begin{align*}
\mathcal{E}(\bar{p},\bar{p}^*) := \|\phi_{\bar{p}}-y \|_m^2 - \|\phi_{\bar{p}^*} -y \|_m^2 \, .
\end{align*}
With the empirical scalar product $\braket{\cdot, \cdot}_m$, which we introduced in Definion~\ref{def:empiricalseminorm}, we decompose the excess risk as
\begin{align} \label{equ:excessdecomposition}
\mathcal{E}(\bar{p},\bar{p}^*) = \|\phi_{\bar{p}} - \phi_{\bar{p}^*} \|_m^2  +  2\braket{\phi_{\bar{p}^*}-y, \phi_{\bar{p}}-\phi_{\bar{p}^*} }_m \, .
\end{align}
We now fix a $\bar{p}^*\in \bar{P}$ and notice, that the minimization problem \eqref{def:Prym} is equal to the minimization of the excess risk.
Hence, we have
\begin{align*}
\mathrm{argmin}_{\bar{p}\in \bar{P}} \|\phi_{\bar{p}}-y \|_m^2 = \mathrm{argmin}_{\bar{p}\in \bar{P}} \|\phi_{\bar{p}} - \phi_{\bar{p}^*} \|_m^2 + 2\braket{\phi_{\bar{p}^*}-y, \phi_{\bar{p}}-\phi_{\bar{p}^*} }_m  \, .
\end{align*}
Since for $\bar{p}=\bar{p}^*$ the excess risk vanishes, at each minimizer of \eqref{def:Prym} the excess risk is less or equal to zero. 
If any $\bar{p}\in\bar{P}$ has a positive excess risk, it thus cannot be a minimizer.
Our strategy to characterize the minimizers of \eqref{def:Prym} consists in proving uniform bounds for positive excess risks at all $\bar{q}\in \bar{P}$, which lead to a generalization error exceeding a chosen threshold.
To this end, we prove lower bounds on both terms in the decomposition \eqref{equ:excessdecomposition} of the excess risk.
We notice, that a lower bound on the minimum of the first term holds uniformly in case of the $\nrip$ event.
A bound on the second term on the right hand side of \eqref{equ:excessdecomposition} follows from the next Lemma.

\begin{lemma}\label{lem:neuralmultiplier}
There exist universal constants $C_0,C_1,C_2,C_3 \in \rr$ such that the following holds: 
Let $\bar{P}$ be a parameter set, that satisfies Assumption~\ref{ass:parameter} with some constants $c_w, c_b$.. 
For any random function $f$ on $\rr^d$, number $m$ of samples, and probability constants $v_1,v_2 \geq C_0$, we have
\begin{align*}
\sup_{\bar{p} \in \bar{P}}  \Big| \braket{f, \phi_{\bar{p}}}_m - \braket{f, \phi_{\bar{p}}} \Big| \leq C_3 v_1v_2 \|f\|_{\psi_2} \frac{n^{\frac{3}{2}} c_w \sqrt{ 8c_b +d+\frac{\ln2}{4} }}{\sqrt{m}}
\end{align*}
with probability at least
\begin{align*}
1- 2\exp\left[-C_1 m v_1^2\right] - 2\exp\left[ -C_2 v_2^2\right] \, .
\end{align*}
\end{lemma}
\begin{proof}
The statement follows from Theorem~4.4 in \cite{mendelson_upper_2016} with $k_0=0$ and the bound on $\Lambda_{0,v_2}(\Phi_{P})$ by Theorem~\ref{the:mendelsonshallow}.
\end{proof}

We now derive upper bounds on the expected risk, which hold uniformly in the sublevel sets \eqref{def:sublevel} of the empirical risk.

\begin{theorem}\label{the:empiricalnoise}
There exist universal constants $C_0,C_1,C_2,C_3,C_4$ and $C_5$, such that the following holds:
Let $\bar{P}$ be a parameter set, that satisfies Assumption~\ref{ass:parameter} with some constants $c_w, c_b$.
We further assume, that for a $\bar{p}^*\in\bar{P}$, a number $m$ of samples, a constant $s \in (0,1)$, precision parameters $t,\omega $ and probability parameters $v_1,v_2>C_0 \, ,u>0$ we have
\begin{align} \label{con:m0}
 m \geq  8 n^3 c_w^2 \left( 8c_b+d + \frac{\ln2}{4}\right) \max \left( C_1 \frac{u}{s\, t^2}  \, , \, C_2 n^2c_w^2 \left(\frac{u}{s \, t^2}\right)^2  \right) 
\end{align}
and
\begin{align}   \label{con:t0}
t \geq \frac{\|\phi_{\bar{p}^*}-y\|_{\mu}}{1-s}  + \sqrt{\frac{\|\phi_{\bar{p}^*}-y\|^2_{\mu}}{(1-s)^2} + C_3 v_1 v_2 \|\phi_{\bar{p}^*}-y\|_{\psi_2} \frac{n^{\frac{3}{2}}c_w \sqrt{8c_b+d + \frac{\ln2}{4}}}{(1-s)\sqrt{m}} + \frac{\omega^2}{1-s}} \, .
\end{align}
Then the probability, that all $\bar{q}\in \bar{P}$ with $\|\phi_{\bar{q}} - y\|^2_m \leq \|\phi_{\bar{p}^*} - y\|^2_m + \omega^2$  satisfy $\|\phi_{\bar{q}} - \phi_{\bar{p}^*} \|_{\mu} \leq t$, is at least
\begin{align*} 
1- 2\exp\left[ -C_4 m v_1^2\right]  - 2\exp\left[ -C_5 v_2^2\right] - 17 \exp\left[ - \frac{u}{4} \right] \, .
\end{align*}
\end{theorem}

\begin{proof}
We first show, that within the stated confidence level we have $\mathcal{E}(\bar{q},\bar{p})>\omega^2$ for all $\bar{q}\in \bar{P}$ with $\|\phi_{\bar{q}}-\phi_{\bar{p}^*}\|_{\mu}>t$. 
Using the decomposition \eqref{equ:excessdecomposition}, the condition $\mathcal{E}(\bar{q},\bar{p})>\omega^2$ is equivalent to
\begin{align} \label{eq:mincondition}
\|\phi_{\bar{p}} - \phi_{\bar{p}^*} \|_m^2  + 2\braket{\phi_{\bar{p}^*}-y, \phi_{\bar{p}} - \phi_{\bar{p}^*}}_m  - \omega^2 >0  \, \text{for all } \bar{p} \text{ with } \|\phi_{\bar{p}} - \phi_{\bar{p}^*}\| > t \, .
\end{align}
We next define the set 
\begin{align*}
\bar{R}_t := \Big\{ \bar{r}:= \frac{(\bar{p},-\bar{p}^*)}{\|\phi_{\bar{p}}  - \phi_{\bar{p}^*} \| } \, : \, \bar{p}\in \bar{P} \, ,  \, \|\phi_{\bar{p}}  - \phi_{\bar{p}^*} \|  > t \Big\} \, .
\end{align*}
If the set $\bar{P}$ satisfies Assumption~\ref{ass:parameter} with parameters $c_w$ and $c_b$, then $\bar{R}_t$ satisfies it with parameters $\frac{c_w}{t}$ and $c_b$.
We then apply Theorem~\ref{thm:NeuRIP2} and obtain that $\nrip(\bar{R}_t)$ holds for $s\in (0,1)$ and $u>0$ with probability at least $1-17 \exp\left[ - \frac{u}{4} \right]$, provided that the number $m$ of samples satisfies
\begin{align} \label{con:m}
m \geq  8 n^3 c_w^2 \left( 8c_b+d + \frac{\ln2}{4}\right) \max \left( C_4 \frac{u}{s\, t^2}  \, , \, C_5 \left(\frac{c_w u}{s \, t^2}\right)^2  \right) \, .
\end{align}
In case $\nrip(\bar{R}_t)$ holds, for all $\| \phi_{\bar{p}} - \phi_{\bar{p}^*}\|_{\mu} > t$, we obtain that
\begin{align} \label{eq:min1}
\|\phi_{\bar{p}} - \phi_{\bar{p}^*} \|_m^2 > (1-s) \|\phi_{\bar{p}} - \phi_{\bar{p}^*} \|_{\mu}^2 \, .
\end{align}
In addition, we apply Lemma~\ref{lem:neuralmultiplier} on $f = \phi_{\bar{p}^*}-y$ and $\bar{P}\times \bar{p}^*$ as the parameter set.
We get that, with probability at least $1- 2\exp\left[ -C_4 m v_1^2\right]  - 2\exp\left[ -C_5 v_2^2\right]$, for all $\bar{q}\in\bar{P}$ we have
\begin{align} 
 \braket{\phi_{\bar{p}^*}-y, \phi_{\bar{q}} - \phi_{\bar{p}^*}}_m & \geq \braket{\phi_{\bar{p}^*}-y, \phi_{\bar{q}} - \phi_{\bar{p}^*}} - C_3 v_1 v_2 \|\phi_{\bar{p}^*}-y\|_{\psi_2} \frac{(2n)^{\frac{3}{2}}c_w \sqrt{8c_b+d + \frac{\ln2}{4}}}{\sqrt{m}} \nonumber \\
& \geq - \|\phi_{\bar{p}^*}-y\|_{\mu} \, \|\phi_{\bar{q}} - \phi_{\bar{p}^*}\|_{\mu}- C_3 v_1 v_2 \|\phi_{\bar{p}^*}-y\|_{\psi_2} \frac{(2n)^{\frac{3}{2}}c_w \sqrt{8c_b+d + \frac{\ln2}{4}}}{\sqrt{m}} \, . \label{eq:min2}
\end{align}
If \eqref{eq:min1} and \eqref{eq:min2} hold, condition \eqref{eq:mincondition} follows from
\begin{align} \label{eq:quadraticineq} 
(1-s) \|\phi_{\bar{q}} - \phi_{\bar{p}^*} \|_{\mu}^2  - 2\|\phi_{\bar{p}^*}-y\|_{\mu} \, \|\phi_{\bar{q}} - \phi_{\bar{p}^*}\|_{\mu}- C_3 v_1 v_2 \|\phi_{\bar{p}^*}-y\|_{\psi_2} \frac{n^{\frac{3}{2}}c_w \sqrt{8c_b+d + \frac{\ln2}{4}}}{\sqrt{m}} > \omega^2 \, . 
\end{align}
For simplification, the constant $C_3$ has here been rescaled with the factor $2^2 \sqrt{2}$, compared to the constant $C_3$ in Lemma~\ref{lem:neuralmultiplier}.
We notice that \eqref{eq:quadraticineq} is a quadratic inequality and solved for all $\|\phi_{\bar{q}} - \phi_{\bar{p}^*}\|_{\mu}>t$, provided that
\begin{align} \label{con:t}
t \geq   \frac{\|\phi_{\bar{p}^*}-y\|_{\mu}}{1-s}  + \sqrt{\frac{\|\phi_{\bar{p}^*}-y\|^2_{\mu}}{(1-s)^2} + C_3 v_1 v_2 \|\phi_{\bar{p}^*}-y\|_{\psi_2} \frac{n^{\frac{3}{2}}c_w \sqrt{8c_b+d + \frac{\ln2}{4}}}{(1-s)\sqrt{m}} + \frac{\omega^2}{1-s}} \,.
\end{align}
We conclude, that in case the inequalities \eqref{con:t} and \eqref{con:m} hold, \eqref{eq:mincondition} is satisfied with probability at least
\begin{align*} 
1- 2\exp\left[ -C_4 m v_1^2\right]  - 2\exp\left[ -C_5 v_2^2\right] - 17 \exp\left[ - \frac{u}{4} \right] \, .
\end{align*}
We have thus shown, that, under the assumptions of the Theorem, we have for all $\bar{q}\in \bar{P}$ with $\|\phi_{\bar{q}}-\phi_{\bar{p}^*}\|_{\mu}>t$ that $\mathcal{E}(\bar{q},\bar{p})>\omega^2$.
If $\mathcal{E}(\bar{q},\bar{p})\leq \omega^2$ holds, as it is the case for $\bar{q}\in \bar{P}$ with  $\|\phi_{\bar{q}} - y\|^2_m \leq \| \phi_{\bar{p}^*} - y\|^2_m + \omega^2$, we have $\|\phi_{\bar{q}}-\phi_{\bar{p}^*}\|_{\mu}\leq t$.
\end{proof}

With Theorem~\ref{the:empiricalnoise} we obtained a bound on the distance of a given $\phi_{\bar{p}^*}$ to any $\phi_{\bar{q}}$, which is parameterized by a $\bar{q}$ in a sublevel set of the empirical risk.
In the following Corollary we use this to derive a bound on the generalization error $\|\phi_{\bar{q}}-y\|_{\mu}$, which holds uniformly for all networks $\phi_{\bar{q}}$ with $\bar{q}\in \bar{Q}_{y,\xi}$.
To facilitate its interpretation, we furthermore introduce a parameter $\alpha$ as a quotient between the number of samples and a term dependent on the network architecture and parameter bounds.

\begin{corollary}\label{cor:agnostic}
There are constants $C_0,C_1,C_2,C_3,C_4$ and $C_5$, such that the following holds: Let $\bar{P}$ satisfy Assumption~\ref{ass:parameter} for some constants $c_w, c_b$ and let $\bar{p}^*\in \bar{P}$ be such that, for some $c_{\bar{p}^*}\geq 0$, we have
\begin{align*}
\mathbb{E}_{\mu}\left[ \exp\left( \frac{(\phi_{\bar{p}^*}(x)-y)^2}{c_{\bar{p}^*}^2} \right) \right] \leq 2 \, .
\end{align*}
We assume for a given $s\in (0,1)$ and confidence parameter $u>0$, that the number $m$ of samples is large enough such that
\begin{align} \label{def:alpha}
\alpha:=\frac{m}{n^3 c_w^2 ( 8 c_b +d + \frac{\ln2}{4})}  \geq 8 \max \left( C_1 \frac{(1-s)^2 u}{s \,\|\phi_{\bar{p}^*}-y \|_{\mu}^2} \, , \, C_2 n^2c_w^2  \left(\frac{ (1-s)^2 u}{s\,\|\phi_{\bar{p}^*} -y\|_{\mu}^2} \right)^2  c_{\bar{p}^*}^2 \right) \, .
\end{align}
We further choose confidence parameters $v_1,v_2>C_0$ and define for some $\omega \geq 0$ the parameter
\begin{align*} 
\eta := \left( \frac{2}{(1-s)} + 1 \right) \|\phi_{\bar{p}^*}-y\|_{\mu} +  \sqrt{\frac{C_3 v_1v_2c_{\bar{p}^*}}{1-s}} \, \alpha^{-\frac{1}{4}} + \frac{\omega}{\sqrt{1-s}} \, .
\end{align*}
If we set $\xi=\sqrt{\| \phi_{\bar{p}^*} - y\|^2_m + \omega^2}$ as the tolerance for the empirical risk, then the probability, that all $\bar{q}\in \bar{Q}_{y,\xi}$ satisfy $\|\phi_{\bar{q}} -y\|_{\mu} \leq \eta$, is at least
\begin{align*}
1- 2\exp\left[ -C_4 m v_1^2\right]  - 2\exp\left[ -C_5 v_2^2\right] - 17 \exp\left[ - \frac{u}{4} \right] \, .
\end{align*}
\end{corollary}
\begin{proof}
For any $\bar{q}\in\bar{P}$, the triangle inequality of the norm $\|\cdot\|_{\mu}$ gives
\begin{align} \label{ine:triangle}
\|\phi_{\bar{q}}-y\|_{\mu}  \leq \|\phi_{\bar{q}}-\phi_{\bar{p}^*}\|_{\mu} + \|\phi_{\bar{p}^*}-y\|_{\mu} \, .
\end{align}
We now apply Theorem~\ref{the:empiricalnoise} to bound the first term on the right hand side of \eqref{ine:triangle}.
To do so, we have to find a parameter $t\geq 0$, for which the conditions \eqref{con:t0} and \eqref{con:m0} are satisfied.
To this end, we take
\begin{align} \label{con:tnew}
t := \frac{2}{(1-s)} \|\phi_{\bar{p}^*}-y\|_{\mu} +  \sqrt{\frac{C_3 v_1v_2c_{\bar{p}^*}}{1-s}} \, \alpha^{-\frac{1}{4}} + \frac{\omega}{\sqrt{1-s}} \, .
\end{align}
We furthermore observe, that $\|\phi_{\bar{p}^*}-y\|_{\psi_2}\leq c_{\bar{p}^*}$, which holds by definition of the Sub-Gaussian norm in \eqref{def:sgnorm}.
With this observation, \eqref{con:tnew} implies condition \eqref{con:t0}.
If furthermore follows from \eqref{con:tnew}, that $t \geq \frac{2}{(1-s)^2} \|y-\phi_{\bar{p}^*} \|_{\mu}$.
This implies that also condition \eqref{con:m0} is satisfied, provided that
\begin{align*}
\alpha \geq  8 \max \left( C_4 \frac{(1-s)^2 u}{s \,\|y-\phi_{\bar{p}^*} \|_{\mu}^2} \, , \, C_5 c_w^2  \left(\frac{ (1-s)^2 u}{s\,\|y-\phi_{\bar{p}^*} \|_{\mu}^2} \right)^2  c_{\bar{p}^*}^2 \right) \, .
\end{align*}
Since the conditions \eqref{con:t0} and \eqref{con:m0} are satisfied, we can apply Theorem~\ref{the:empiricalnoise} to achieve a bound on the first term of the right hand side of \eqref{ine:triangle} by $t$, which holds uniformly for all $\bar{q}\in\bar{Q}_{y,\xi}$.
We further observe, that $\eta=t+\|\phi_{\bar{p}^*}-y\|_{\mu}$.
In combination with \eqref{ine:triangle}, this yields the bound $\|\phi_{\bar{q}}-y\|_{\mu} \leq \eta$ for all $\bar{q}\in\bar{Q}_{y,\xi}$.
\end{proof}

\end{document}